\documentclass[twocolumn]{article}

\usepackage{xcolor}
\usepackage{cite}
\usepackage{amsmath,amssymb,amsfonts,bm}
\usepackage{amsthm}
\usepackage{algorithmic}
\usepackage{graphicx}
\usepackage{textcomp}
\usepackage{subcaption}
\usepackage{setspace}

\newcommand{\dlt}{\textsc{Glimpse}}

\def\x{{\mathbf{x}}}
\def\A{{\mathbf{A}}}
\def\n{{\mathbf{n}}}

\def\f{{\mathbf{f}}}

\newcommand{\thetab}{\mathrm{\bm{\theta}}}

\def\s{{\mathbf{s}}}
\def\h{{\mathbf{h}}}

\def\p{{\mathbf{p}}}
\def\kk{{\mathbf{k}}}
\newcommand{\xib}{\mathrm{\bm{\xi}}}

\newcommand{\R}{\mathbb{R}}

\DeclareMathOperator*{\argmin}{argmin}

\newtheorem{lemma}{Lemma}

\newtheorem{proposition}[lemma]{Proposition}

\usepackage{hyperref}
\definecolor{mydarkgreen}{rgb}{0,0.6,0.30}
\definecolor{mydarkblue}{rgb}{0,0.30,0.65}
\hypersetup{
    colorlinks=true,    
    linkcolor=mydarkblue,     
    citecolor=mydarkgreen,    
    urlcolor=mydarkblue       
}

\definecolor{mydarkred}{rgb}{0.6,0.0,0.30}
\newcommand{\rev}[1]{{\color{black} #1}}
\newcommand{\revv}[1]{{\color{black} #1}}

\usepackage[textsize=tiny]{todonotes}
\usepackage{geometry}
\makeatletter
\renewcommand{\@fnsymbol}[1]{}
\makeatother

\title{{\dlt}: Generalized Locality for Scalable and Robust CT}

\author{AmirEhsan Khorashadizadeh, Valentin Debarnot, Tianlin Liu, and Ivan Dokmanić}

\begin{document}
\date{}


\maketitle

\begin{abstract}
Deep learning has become the state-of-the-art approach to medical tomographic imaging. A common approach is to feed the result of a simple inversion, for example the backprojection, to a multiscale convolutional neural network (CNN) which computes the final reconstruction. Despite good results on in-distribution test data,
this often results in overfitting certain large-scale structures and poor generalization on out-of-distribution (OOD) samples. Moreover, the memory and computational complexity of multiscale CNNs scale unfavorably with image resolution, making them impractical for application at realistic clinical resolutions. 
In this paper, we introduce {\dlt}, a \textit{local} \rev{coordinate-based} neural network for computed tomography which reconstructs a pixel value by processing only the measurements associated with the neighborhood of the pixel. {\dlt} significantly outperforms successful CNNs on OOD samples, while achieving comparable or better performance on in-distribution test data and maintaining a memory footprint almost independent of image resolution; 5GB memory suffices to train on 1024 $\times$ 1024 images which is orders of magnitude less than CNNs. {\dlt} is fully differentiable and can be used plug-and-play in arbitrary deep learning architectures, enabling feats such as correcting miscalibrated projection orientations. Our implementation and Google Colab demo can be accessed at \url{https://github.com/swing-research/Glimpse}.
\end{abstract}

\vspace{1em}

\section{Introduction}
\label{sec:introduction}
Convolutional neural networks (CNNs) have become the standard approach for tomographic image reconstruction~\cite {wang2020deep}. The U-Net~\cite{ronneberger2015u} architecture underpins numerous deep learning reconstruction methods, achieving strong results on a variety of imaging problems including computed tomography (CT)~\cite{Jin2017deep}, magnetic resonance imaging (MRI)~\cite{mccann2017convolutional} and photoacoustic tomography~\cite{davoudi2019deep}. Its success is often attributed to the particular multi-scale architecture~\cite{Liu2022multiscale}.

\revv{
Despite remarkable progress with CNN-based methods, some core practical challenges complicate their application to real problems:}
\begin{itemize}
    \item \revv{\textbf{Poor Generalization under Distribution Shift:}} CNNs show good performance on in-distribution test images similar to the training data but tend to overfit class-specific image content. This results in poor robustness to distribution shift in data and sensing \rev{\cite{li2022noise, antun2020instabilities}}. \textit{Model-based} networks address this drawback by integrating the forward and adjoint operators into multiple network layers or iterations~\cite{Adler2017solving,  Adler2018learned, gilton2019neumann,maier2019learning,Hauptmann2020multi, sahel2021deep}. This, however, hurts scalability.
    \item \revv{\textbf{High Memory and Computation Cost:}} The required memory grows steeply with image resolution~\cite{Leuschner2021quantitative} for CNNs and even more steeply for model-based networks such as learned primal-dual (LPD)~\cite{Adler2018learned}. \rev{Moreover, unlike standard networks like U-Net which can handle large images by working on patches, model-based networks like LPD do not permit patch processing since the Radon transform in the network does not handle incomplete data.}
\end{itemize}

\begin{figure*}
    \centering
    \includegraphics[width = \textwidth]{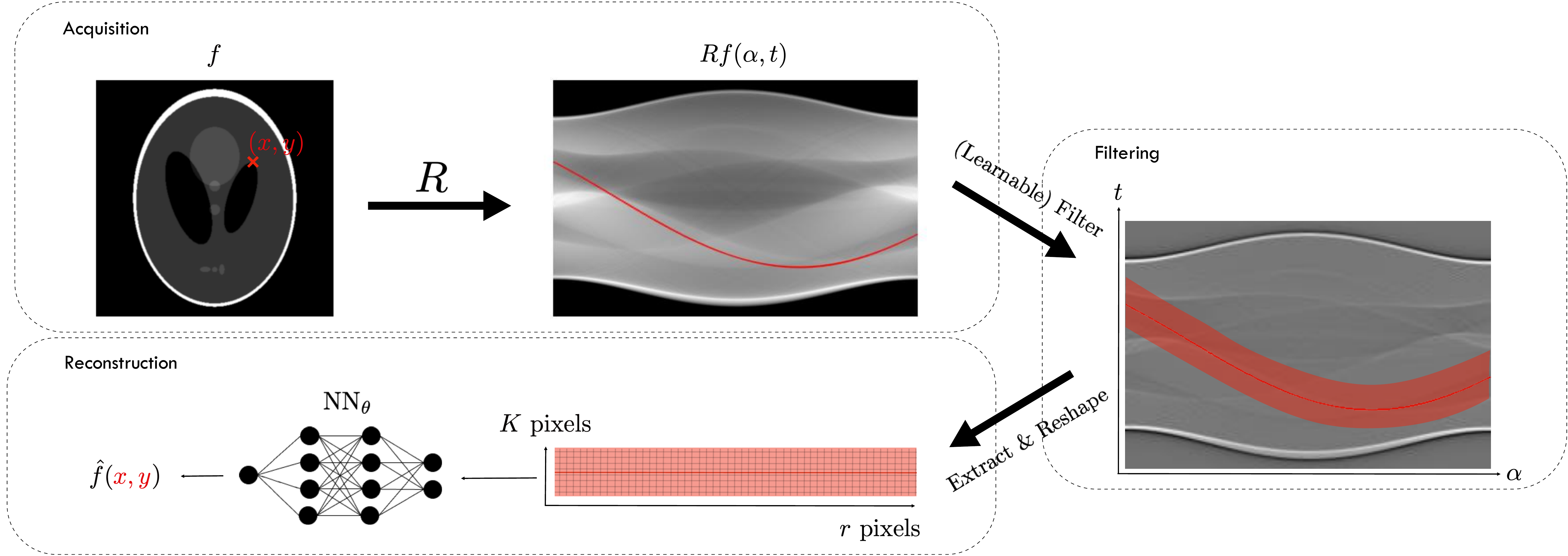}
    \caption{\dlt; $\text{NN}_\theta$ processes the measurements associated with the pixel $(x,y)$ and its neighbors extracted from the filtered sinogram. This local processing network has promising performance on OOD data while being computationally efficient all due to its locality.}
    \label{fig: network}
\end{figure*}

\subsection{\revv{Our Innovations}} In this paper, we propose \dlt, a novel coordinate-based local reconstruction framework for sparse-view CT. As shown in Figure~\ref{fig: network}, unlike large-scale CNNs that operate globally on filtered backprojection (FBP)~\cite{feldkamp1984practical} reconstructions, \dlt~estimates a given pixel value using only \emph{local measurements in the sinogram domain} associated with this pixel. There is no backprojection step. {Localization prevents {\dlt} from overfitting the large-scale features and results in robust performance under distribution shift.}

\begin{figure*}
    \centering
    \begin{subfigure}{0.5\textwidth}
      \centering
     \includegraphics[width=1.05\textwidth]{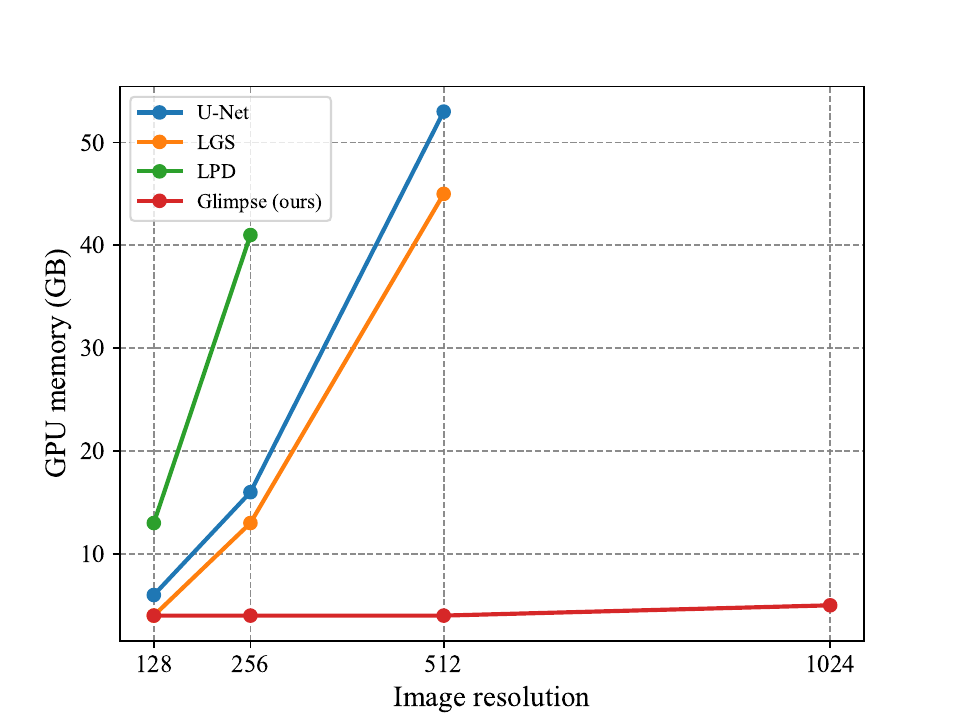}
    \caption{Memory footprint  (batch size 64)}
    \label{fig: gpu_train}
    \end{subfigure}%
    \begin{subfigure}{0.5\textwidth}
    \centering
    \includegraphics[width=1.05\textwidth]{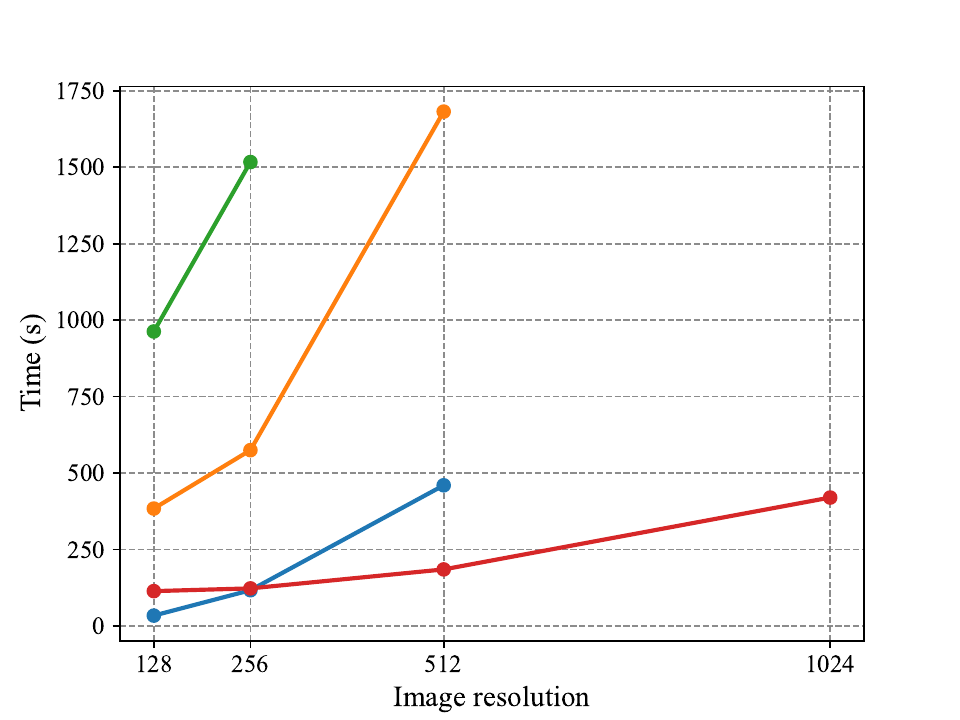}
    \caption{Training time (500 iterations)}
    \label{fig: time_train}
    \end{subfigure}
    \caption{The memory and time requirements during training vary across different models, with {\dlt} being substantially faster and more memory-efficient compared to the baselines. Remarkably, {\dlt}'s memory usage remains nearly constant regardless of image resolution, making it an excellent choice for high-dimensional image reconstruction tasks. All experiments were performed on a single A100 GPU with 80GB of memory. Missing data points indicate that the corresponding model exceeded the GPU's memory capacity at the specified resolution.}
    \label{fig: training computation}
\end{figure*}

At the same time, it results in \textit{high computational efficiency}: the coordinate-based design permits training on mini-batches of both \emph{pixels} and objects. This leads to fast and efficient training, requiring a small, fixed amount of memory almost independent from the image resolution. {As shown in Figure \ref{fig: training computation}, {\dlt} requires significantly less memory and training time than CNNs, in particular compared with model-based networks like LPD.} It can efficiently train on realistic images in resolution $1024 \times 1024$ and beyond.

{\dlt} is fully differentiable, all the way down to the sensing and integration geometry. This is an advantage over the standard CNN-based architectures. Most approaches to CT rely on fixed sensor geometry which is encoded in the forward operator, whether explicitly, as seen in methods like FBP~\cite{feldkamp1984practical}, SART~\cite{andersen1984simultaneous}, LGS~\cite{Adler2017solving}, and LPD~\cite{Adler2018learned} or implicitly in U-Net~\cite{ronneberger2015u} when taking FBP as input. This fixed geometry is a problem when faced with uncertainties in calibration or blind inversion problems where the sensor geometry information is entirely unavailable \cite{lunz2021learned, gupta2023differentiable}. Our differentiable architecture allows us to estimate projection angles which results in better reconstructions. Furthermore, differentiability enables us to replace the fixed FBP filter by one that is optimal for the noise level and data distribution; this is illustrated in Figure \ref{fig: network}. All this ultimately results in high-quality reconstructions.

\subsection{\revv{Why are U-Nets Sensitive to Distribution Shift?}}

\begin{figure}
    \centering
    \includegraphics[width = 0.48\textwidth]{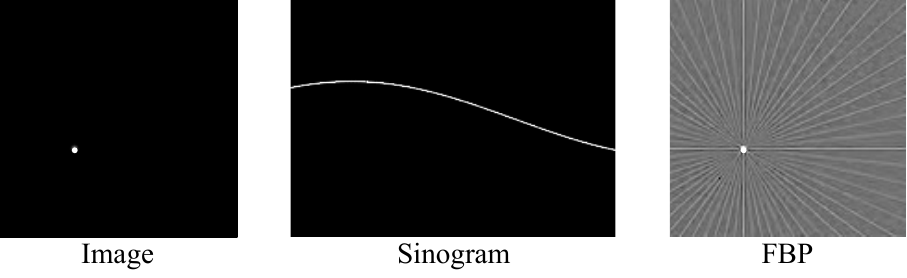}
    \caption{A point source image, its sinogram, and the sparse view FBP reconstruction. While the corresponding measurements for this pixel have sinusoidal support in the sinogram, this information is diffused all over the FBP image. \emph{The contrast of the FBP image has been stretched to emphasize this effect.}}
    \label{fig: single_pixel_sinogram}
\end{figure}

\begin{figure}
    \centering
    \includegraphics[width = 0.47\textwidth]{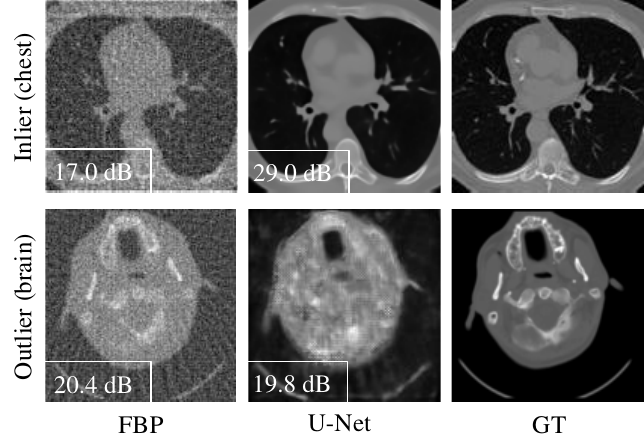}
    \caption{Performance of U-Net~\cite{ronneberger2015u} trained on chest images in resolution $128 \times 128$: evaluation on in-distribution test data (chest samples) and OOD brain samples shows that the large receptive field of U-Net hinders its ability to generalize on OOD samples, with its PSNR even falling below that of FBP reconstruction. \rev{We indicate PSNRs between the reconstructions and the ground truth.}}
    \label{fig: U-Net}
\end{figure}

We close the introduction by presenting an experiment which illustrates why U-Net-like CNNs---which post-process FBP reconstructions---generalize poorly out-of-distribution. Figure~\ref{fig: single_pixel_sinogram} shows a point-like object, its sparse view sinogram, and the FBP reconstruction. It is evident that the FBP is supported over the entire field of view. This raises the question of the ideal receptive field size for CNNs like U-Net: a large receptive field is statistically beneficial to gather information correlated with the value of a target pixel~\cite{aggarwal2018modl, hamoud2022beyond}.
\rev{A similar argument shows that backprojection introduces long-range correlations in noise.}

But the issue with models with large receptive fields is that they often overfit class-specific image content in training data which leads to poor generalization on out-of-distribution samples \cite{khorashadizadeh2022funknn}. Indeed, Figure~\ref{fig: U-Net} shows that while U-Net produces good results when tested on chest images similar to training data, it performs poorly on out-of-distribution brain images. This is problematic in domains such as medical imaging where robustness over distribution shifts and other uncertain and variable factors is important~\cite{graas2023just}.

\section{Related Work}
\label{sec: Related Works}

\subsection{Model-based vs Model-free Inversion}
\label{sec: model_free_model_based}
There are two major classes of deep learning to CT reconstruction: \emph{model-based} and \emph{model-free}.
In the model-based approach, neural networks process raw sinograms and map them to the desired CT images while the Radon transform is integrated into multiple network layers or iterations~\cite{Adler2017solving, Adler2018learned, Hauptmann2020multi,ronchetti2020torchradon}.
These methods perform remarkably well across various inverse problems, but they are computationally expensive, especially during training \cite{Leuschner2021quantitative}. The high computational cost is due, among other factors, to the repeated application of the Radon transform and its adjoint in the network architecture. 

By contrast, model-free approaches offer a computationally cheaper alternative. The Radon transform (or its adjoint) is only used once in FBP computation before the neural network~\cite{Jin2017deep, Kang2017deep,khorashadizadeh2023conditional}. 
However, these models often require deep networks with a large receptive field to leverage the information delocalized across the FBP image. \rev{Recent studies aim to bypass the fixed FBP operator to provide greater flexibility. The common approach is a direct sinogram-to-image mapping that combines CNNs and MLP blocks, effectively replacing the FBP operator with learnable components \cite{li2019learning, wurfl2018deep}. He et al. \cite{he2020radon} present a partially learnable FBP by substituting the traditional Ram-Lak filter with an MLP block and incorporating learnable weighted averaging in the backprojection step. This modified FBP is further refined by a post-processing CNN.}
Recently, Hamoud et al. \cite{hamoud2022beyond} used a measurement rearrangement technique to stratify backprojected features by angle and thus enable the use of smaller, shallower CNNs. 

\rev{
\subsection{Robustness of deep learning for image reconstruction}
\label{sec: robustness}
As discussed in Section \ref{sec:introduction}, deep neural networks often suffer from poor generalization and unstable reconstructions \cite{antun2020instabilities, raj2020improving, li2022noise}. In \cite{colbrook2022difficulty}, the authors present a theoretical study that highlights a trade-off between stability and accuracy and propose neural networks that navigate this trade-off and improve generalization. Genzel et al. study the role of network architecture in improving generalization \cite{genzel2022solving}. Incorporating the forward operator and enforcing measurement consistency have been shown to substantially improve generalization \cite{aggarwal2018modl, Adler2018learned, wu2024multi}. Another technique to improve generalization is jittering by additive Gaussian noise during training \cite{genzel2022solving, krainovic2024learning}. In this paper, we show that computationally efficient neural networks which incorporate the right notion of transform-domain locality achieve excellent generalization in- and out-of-distribution.
}

\rev{
\subsection{Implicit Neural Representation for Imaging}
\label{sec: INR}
{\dlt} is a coordinate-based reconstruction framework that recovers the image intensity at each pixel separately. Recently, neural fields, also known as implicit neural representations (INRs) \cite{sitzmann2020implicit, atzmon2020sal, chabra2020deep}, have emerged as a promising coordinate-based approach for representing continuous signals, images, and 3D volumes. Unlike traditional deep learning models that represent signals as discrete arrays, INRs use deep neural networks, typically MLPs, to map coordinates to signal values, enabling a \textit{continuous} signal representation. This approach offers several advantages over conventional models. For instance, INRs can seamlessly interpolate signals within a continuous space instead of being limited to a single resolution. Moreover, their coordinate-based representation allows for flexible memory usage, making them particularly well-suited for high-dimensional 3D reconstructions \cite{chen2019learning, peng2020convolutional, jiang2020local, dupont2022data, dupont2021generative, susmelj2024uncertainty} and scene representations \cite{mildenhall2021nerf}.

Coordinate-based models have also demonstrated strong performance in computational imaging. INRs efficiently model signals and their spatial derivatives which is useful for solving partial differential equations (PDEs) \cite{sitzmann2020implicit, vlavsic2022implicit}. They can be combined with self-supervised learning to learn a continuous representation of sub-sampled CT sinograms \cite{sun2021coil}. Zha et al. \cite{zha2022naf} use INRs to learn a continuous image representation that aligns with sinogram measurements for cone-beam CT reconstruction. Unlike all these methods, {\dlt} learns a map from \textit{both measurements and coordinates} to reconstruction values at individual pixels and is thus a true, learned image reconstruction operator rather than a signal parameterization.}

\subsection{Uncalibrated CT Imaging.}
\label{sec: uncalibrated discussion}
In CT imaging, the acquisition operator is usually known but only a limited number of measurements is collected, either to minimize radiation exposure or shorten acquisition time (sparse view) or when sample geometry and stage mechanics limit projection angles to a cone (limited view).
In certain situations, the acquisition operator is only partially or approximately known. Neglecting this uncertainty can result in a significant drop in the quality of the reconstructions~\cite{lunz2021learned}. To tackle this challenge, total least squares approaches have been developed, involving the perturbation of an assumed forward operator~\cite{golub1980analysis, markovsky2007overview, gupta2021total} or trained networks combined with autodifferentiation and resampling \cite{gupta2023differentiable}.

\section{Methods}
\label{sec: methods}
\revv{In this section we introduce {\dlt
}. We begin with a brief overview of tomographic imaging in order to introduce the filtered backprojection formula.}
\subsection{Computed Tomography}
\label{sec: Computed Tomography}
Tomographic imaging~\cite{kak2001principles} plays an important role in many applications including medical diagnosis~\cite{wang2008outlook}, industrial testing~\cite{de2014industrial}, and security~\cite{wells2012review}.
We consider 2D computed tomography where the image of interest $f(\x)$ with size $D \times D$ is reconstructed from measurements of (X-ray) attenuation. The forward model is the Radon transform $Rf$ which computes integrals of $f(\x)$ along lines $L$, 
\begin{align}
    Rf(L) = \int_{L} f\big(\x \big) |d\x|.
    \label{eq: Radon with L}
\end{align}
We parameterize a line $L$ by its distance from the origin $t$ and its normal vector's angle with the $x$-axis $\alpha$. We can then reformulate~\eqref{eq: Radon with L} as 
\begin{align} \label{eq:radon_def}
    &Rf(\alpha, t) = \int_{-\infty}^\infty f\big(x(z), y(z) \big) dz,
\end{align}
where,
\begin{align}
    x(z) &= z\cos(\alpha) - t\sin(\alpha), \\
    y(z) &= z\sin(\alpha) + t\cos(\alpha).
\end{align}
The image of interest is observed from a finite set of $r$ different viewing directions $\{\alpha_m\}_{m = 1}^r$, each having $N$ parallel, equispaced rays. The measurements of the attenuation are then represented as a transform-domain ``image'' $\s \in \R^{N \times r}$ called a sinogram.

Standard methods for CT image recovery discretize the image of interest $f(\x)$ into a discrete image  $\f \in \R^{N \times N}$ supported on an $N \times N$ grid. 
After discretization, the forward model can be written as
\begin{equation}
    \s = \A \f + \n
    \label{eq: forward_op}
\end{equation}
where $\A$ is the matrix of the discretized Radon transform and we model the measurement noise by $\n$. The most commonly used analytical inversion method is the filtered backprojection (FBP),
\begin{align}
    \f^\text{FBP}_{x,y} = \sum_{m = 1}^r \tilde{\s}(y\cos(\alpha_m) - x\sin(\alpha_m),m),
    \label{eq: FBP}
\end{align}
where $\f^\text{FBP} \in \R^{N \times N}$ is the FBP reconstruction, $\tilde{\s}[\cdot,m] = \s[\cdot,m] * \h$, $\h$ is a certain high-pass filter, $*$ denotes the convolution and linear interpolation is used in \eqref{eq: FBP} for evaluating $\tilde{\s}(x,\cdot)$ when $x$ is not an integer. As shown in Proposition~\ref{prop:optimal_filter} in Appendix~\ref{app: optimal filter}, while the Ram-Lak filter is the optimal choice for $\h$ in the case of noise-free complete measurements, it amplifies noise in real measurements, yielding poor reconstructions.

With noise and an incomplete collection of projections, tomographic image reconstruction is an ill-posed inverse problem that requires an image prior as regularizer.
We introduce our proposed method, {\dlt}, designed to respect the geometry of CT, which implicitly learns such a prior from training data.

\subsection{{\dlt}: Generalized Local Imaging with MLPs}
\label{sec: model}

To recover the image $\f(x,y)$ at location $\x = (x,y)$, we identify the elements in the sinogram $\s$ influenced by this pixel. As illustrated in Figure~\ref{fig: single_pixel_sinogram}, the corresponding measurements for the pixel $(x,y)$ are supported along a sinusoidal curve in the sinogram; we denote them $\text{SIN}_{x,y} \in \R^r$, with elements being given as 
\begin{align}
    \text{SIN}_{x,y}(m) = \s(y\cos(\alpha_m) - x\sin(\alpha_m),m).
    \label{eq: sinogram sampler}
\end{align}
Similar to \eqref{eq: FBP}, we can use interpolation to evaluate $\s(x,\cdot)$ for non-integer $x$. This localization is formally captured by the following proposition.
\begin{proposition}[Impulse response of Radon transform]\label{prop:sine}
    Let $f(u,v) = \delta(u-x,v-y)$ be the Dirac delta distribution in $\R^2$ at location $(x, y)$. Its Radon transform (in the sense of distributions) is
    \begin{align*}
        Rf(\alpha, t)
         &= \begin{cases}
      1, & \text{if}\ t = r \cos(\alpha + \varphi) \\
      0, & \text{otherwise},
    \end{cases}\\[-2mm]
    \end{align*}
    where $r = \sqrt{x^2 + y^2}$, $\varphi = \mathrm{atan2}(y, x)$, and $\mathrm{atan2}(\cdot, \cdot)$ the four-quadrant arctangent.
\end{proposition} 
The standard proof is outlined in Appendix~\ref{app:proof_sine}.

This may seem to suggest that the neighborhood of the sinusoid-shaped part of the sinogram $\text{SIN}_{x,y}$ contains sufficient information to recover the pixel intensity at location $(x,y)$. Note however that the pixel at $(x, y)$ influences the integral over any line passing through it and thus also the parts of the sinogram corresponding to pixels on those other lines; this can be loosely thought of as a consequence of non-orthogonality of the Radon transform. The above statement is thus more accurately a statement about the \textit{filtered} sinogram since high-pass filtering in the FBP ``relocalizes'' information. We mention in the passing that it is also related to the celebrated support theorems of Sigurdur Helgason, Jan Boman, and others \cite{helgason1965radon,HELGASON198091,boman1987support,boman2006helgason} which state that a compactly-supported image may be recovered from a compactly-supported subset of its Radon data under idealized sampling and SNR conditions. 

Indeed, the high-pass filtering in the FBP is derived for noiseless data and a continuum of observed angles. In reality, the projections are corrupted with noise and come from a sparse subset of projection angles. We address this by 1) incorporating ``contextual information'' about the target pixel and 2) letting the filter be learnable to adapt it to the specifics of discretization and noise.

As shown in Figure~\ref{fig: network}, we exploit the spatial regularity of medical images (encoded in training data) by using the measurements that provide \textit{local} information around $(x,y)$. This ensures that the model does not overfit large-scale features in the training data while maintaining low computational complexity. We thus additionally extract from the filtered sinogram the regions associated with the neighboring pixels around $(x,y)$ and store this information in vector $\p_{x,y}$,
\begin{equation}
    \p_{x,y} = \{ \text{SIN}_{x+dn,y+dn'} | n,n' = -\left\lfloor C/2 \right\rfloor,\cdots,\left\lfloor C/2 \right\rfloor \},
    \label{eq: cropper}
\end{equation}
where $K = C^2$ determines the number of neighboring pixels around $(x,y)$ for an odd number $C\geq 1$ and $d$ denotes the scale of the window which adjusts the receptive field. 
In order to recover the image at location $(x,y)$ from $\p_{x,y}$, we use a neural network $\text{NN}_\theta:~\R^{r \times K} \to~\R$ with parameters $\theta$,
\begin{equation}
    \hat{\f}(x,y) = \text{NN}_\thetab\big(\p_{x,y}\big),
    \label{eq: pipeline}
\end{equation}
which estimates the pixel intensity $\hat{\f}_{x,y}$ from the local features around $(x,y)$.
As we typically use a small neighborhood size $K$, we can parameterize $\text{NN}_\theta$ using a multi-layer perceptron (MLP).
We call the proposed model {\dlt}, standing for generalized\footnote{The word ``generalized'' emphasizes that locality is also encoded in the transform domain, not just in real space as in some of earlier work.} local imaging with MLPs. \rev{\rev{{\dlt} can be viewed as a learnable alternative to FBP as it replaces the simple averaging along the corresponding sinusoidal support with a learnable non-linear operator, parameterized by $\text{NN}_\theta$, which processes the local contextual measurements. Our method can be seen as an interpolation between CNNs applied globally to FBP reconstructions and model-based architectures which explicitly employ the backprojection operator. This is because our inversion is structured "like an FBP" (which simply sums filtered sinogram values along the sinusoidal support) whereas we allow for a more general function of the neighborhood of the sinusoidal support (and thus can approach optimal reconstruction for a larger class of priors than Gaussian processes).}}

In the following section, we provide further details regarding {\dlt}'s architecture. We describe in Section \ref{sec: adaptive filter} how our implementation of {\dlt} allows adapting to noisy measurements. We then propose a training strategy with resolution-agnostic memory usage in Section~\ref{sec: Training}. In Appendix~\ref{app: learnable geometry}, we show how backpropagating through {\dlt} can compensate for calibration errors.

\subsection{MultiMLP: efficient processing of increased projections}
\label{sec: MultiMLP}
\begin{figure}
    \centering
    \includegraphics[width = 0.43 \textwidth]{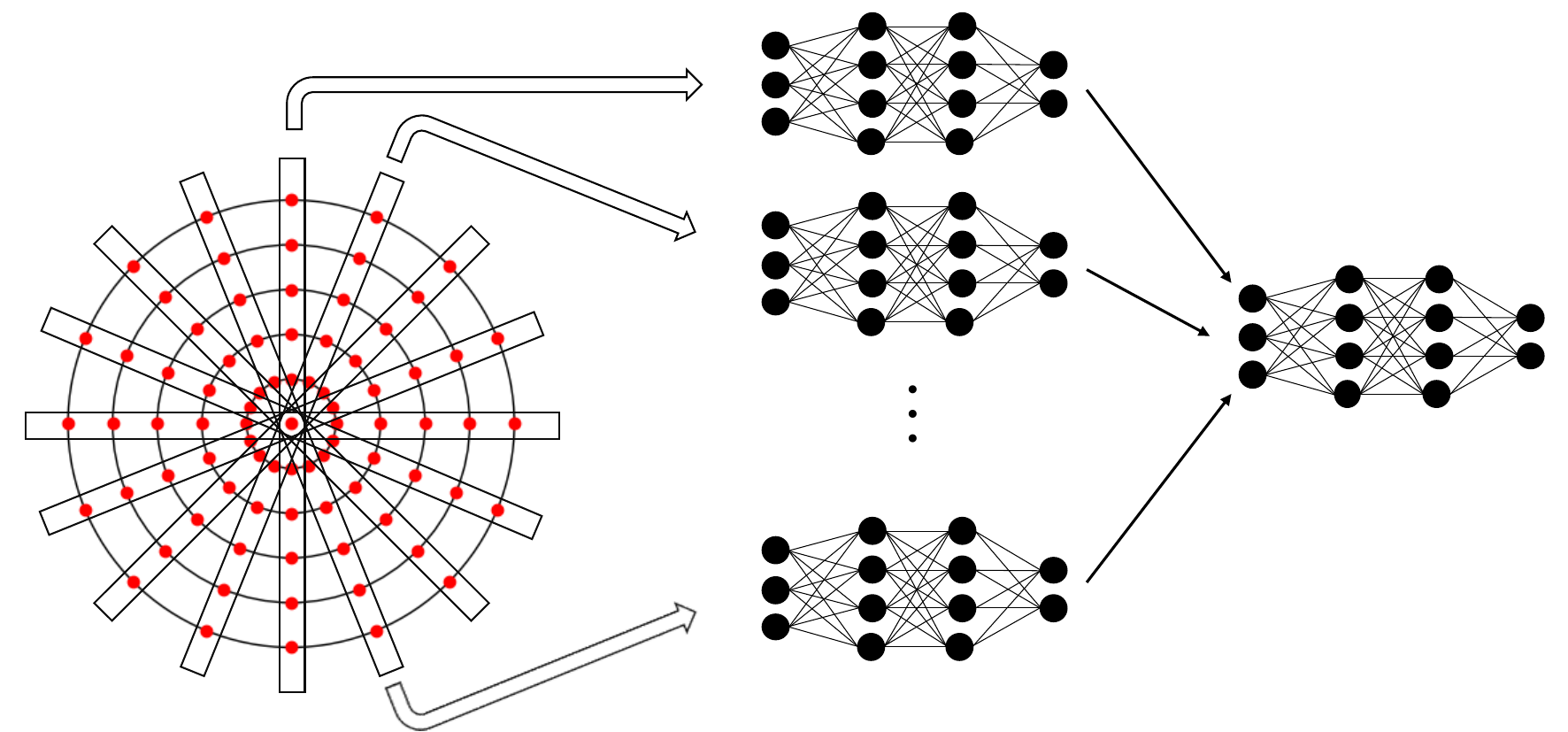}
    \caption{MultiMLP architecture; the input patch (here over a circular geometry) is split into smaller chunks each processed with a separate MLP, the extracted information is then mixed by another MLP. Each red point contains the associated sinusoidal curve extracted from the sinogram.}
    \label{fig: multiMLP}
\end{figure}
\rev{
The number of parameters in $\text{NN}_\theta$ when parameterized by an MLP scales with the number of projections $r$ and the neighborhood size $K$, which increases computational complexity and slows down training. To mitigate this issue, we propose MultiMLP, a new architecture designed to efficiently process large numbers of projections and neighborhoods. Inspired by vision transformers \cite{dosovitskiy2020image}, we partition the extracted measurements $\p_{x,y}$ into smaller chunks, each processed by a separate MLP, as illustrated in Figure \ref{fig: multiMLP}. The outputs of these MLPs are then mixed using another MLP. For ease of visualization, we show a circular neighborhood where each red point represents its associated sinusoidal curve.
}

\subsection{Adaptive Filtering for Noisy Measurements}
\label{sec: adaptive filter}

The Ram-Lak high-pass filter is the optimal filter $\h$ for the FBP reconstruction in the case of complete noise-free measurements; see Appendix~\ref{app: optimal filter} for a standard demonstration. In real applications, however, we always encounter noisy projections from a subset of angles. The Ram-Lak filter is then suboptimal and typically degrades the reconstruction quality as it amplifies high-frequency noise. Alternative filters with lower amplitudes in high frequencies like Shepp-Logan, Cosine, and Hamming have been used to address this, but they are all ad hoc choices. It is advantageous to adapt $\h$ to the specifics of noise and sampling strategy in the target application. To design this task-specific filter, 
we consider the coefficients of the filter $\h$ (in the frequency domain) as trainable parameters to be optimized during training as depicted in Figure \ref{fig: network}. This allows us to automatically learn a noise-adaptive filter from data, again with almost no additional computational cost.

\subsection{Resolution-agnostic Memory Usage in Training}
\label{sec: Training}
\rev{{\dlt} is fully differentiable which enables the optimization of the receptive field scale, filter parameters, and MLP weights via backpropagation during training.} To simplify notation, we denote the entire described {\dlt} pipeline by $\hat{\f}(\x) =  \text{\dlt}_\phi(\x,\s)$. The inputs are the target pixel coordinates $\x = (x,y)$ and the sinogram $\s$; the output is an estimate of $\f(x, y)$. The parameters $\phi$ denote the trainable parameters of {\dlt} including the MLP weights $\theta$, the projection angles $\{\alpha_m\}_{m=1}^r$ (see Appendix~\ref{app: learnable geometry}), the adaptive filter $\h$ and the window receptive field scale $d$. We consider a set of training data $\{(\s_i, \f_i)\}_{i = 1}^L$ from the noisy sinograms and images. We optimize the {\dlt} parameters $\phi$ using gradient-based optimization by minimizing
\begin{align}
    \phi^* = \argmin_\phi ~ \sum_{i = 1}^{N^2} \sum_{j = 1}^L | \text{\dlt}_\phi (\x_i, \s_j) - \f_j(\x_i) |^2.
    \label{eq: training}
\end{align}
At inference time, we simply evaluate the image intensity at any pixel as $\hat{\f}_\text{test}(\x) = \text{\dlt}_{\phi^*}(\x, \s_\text{test})$.
One major advantage of {\dlt} compared to CNNs like U-Net and LPD is its low memory and compute complexity. Memory requirements of CNN-based models scale steeply with image resolution, making them prohibitively expensive for realistic resolutions.
As shown in \eqref{eq: training}, {\dlt} can be trained using stochastic gradient-based optimizers with the flexibility to select mini-batches from both the objects and pixels \rev{thanks to its coordinate-based design}. This leads to a memory footprint nearly agnostic to resolution, which makes {\dlt} suitable for training on realistic image resolutions of $1024 \times 1024$ and higher.

\section{Experiments}
\label{sec: Experiments}

We benchmark {\dlt} against successful CNN-based baselines for sparse-view CT reconstruction: U-Net~\cite{ronneberger2015u}, \rev{iRadonMAP \cite{he2020radon} with U-Net as the post-processing CNN}, learned gradient scheme (LGS)~\cite{Adler2017solving} and learned primal-dual (LPD)~\cite{Adler2018learned}. \rev{For a thorough comparison we created two additional baselines: 1) iRadonMAP-ff: in the original iRadonMAP, the filter $\h$ in \eqref{eq: FBP} is replaced with an MLP architecture. Here, we consider iRadonMAP-ff which rather uses the learnable Fourier filter $\h$ introduced in Section \ref{sec: adaptive filter}, allowing us to ablate the effects of different filtering procedures; 2) iRadonMAP-ffnu: the original iRadonMAP employs a post-processing CNN to enhance reconstruction quality. To assess the performance of the linear model alone, we consider iRadonMAP-ffnu, which excludes the CNN. This comparison with {\dlt} helps us understand the significance of our non-linear mapping $\text{NN}_\theta$ and the inclusion of neighboring pixels.}
The reconstruction quality is quantified using the peak signal-to-noise ratio (PSNR) and Structural Similarity Index (SSIM)~\cite{wang2004image}. \rev{Bottom left windows in Figures show the PSNR between the reconstructed image and the ground truth.}

We implement all models in PyTorch~\cite{paszke2019pytorch} on a machine equipped with a Nvidia A100 GPU with 80GB memory. All models were trained for 200 epochs with MSE loss using the Adam optimizer~\cite{kingma2014adam}. A learning rate of $10^{-4}$ was used for {\dlt}, U-Net and iRadonMAP, and of $10^{-3}$  for LGS and LPD. All models were trained with batch size 64. For {\dlt}, for each mini-batch of random targets, we ran optimization on a random mini-batch of 512 pixels 3 times.

In Section~\ref{sec: Sparse-view CT}, we compare {\dlt} to CNN-based models for sparse-view CT  reconstruction on both in-distribution and OOD data. In Section~\ref{sec: computational efficiency}, we analyze the computational efficiency of the aforementioned models. We analyze the learned filters $\h$ across different measurement noise levels in Section~\ref{sec: learned filter}. \rev{We study the influence of the number of projections and neighboring pixels in Sections \ref{sec: ablation_proj} and \ref{sec: window size}.} 
\rev{Finally, in Appendix~\ref{app: learnable geometry}, we present our method for learning the projection angles jointly with the image reconstruction to address uncalibrated and blind scenarios.}

\subsection{Sparse view CT Image Reconstruction}
\label{sec: Sparse-view CT}

\begin{figure*}
    \centering
    \begin{subfigure}{.95\textwidth}
      \centering
     \includegraphics[width=\textwidth]{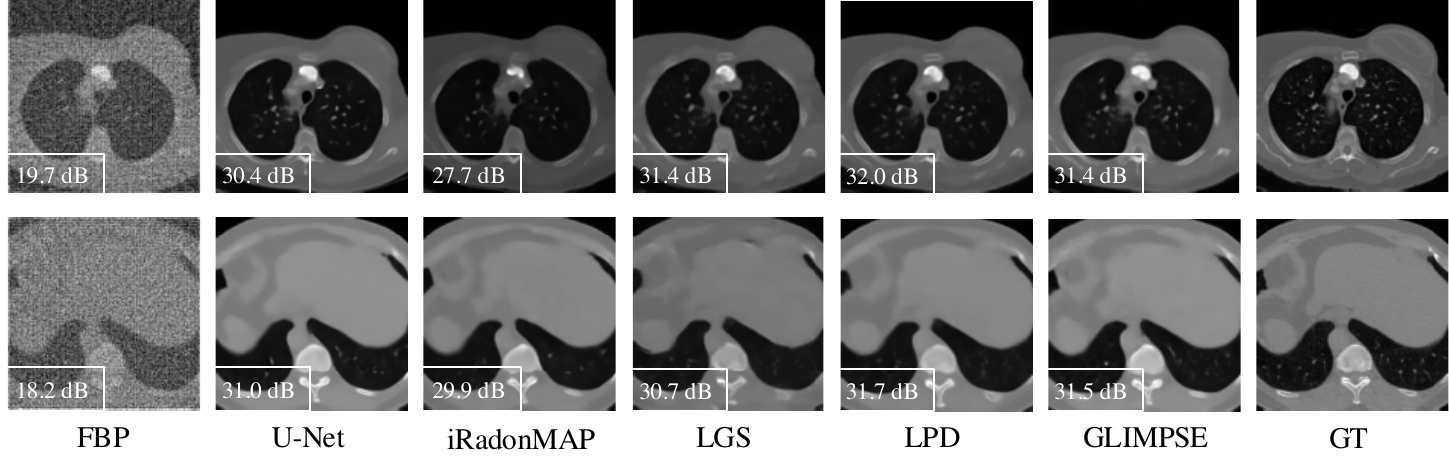}
    \caption{In-distribution chest samples}
    \label{fig: 128_test}
    \end{subfigure} \\ [3mm]
    \begin{subfigure}{.95\textwidth}
    \centering
    \includegraphics[width=\textwidth]{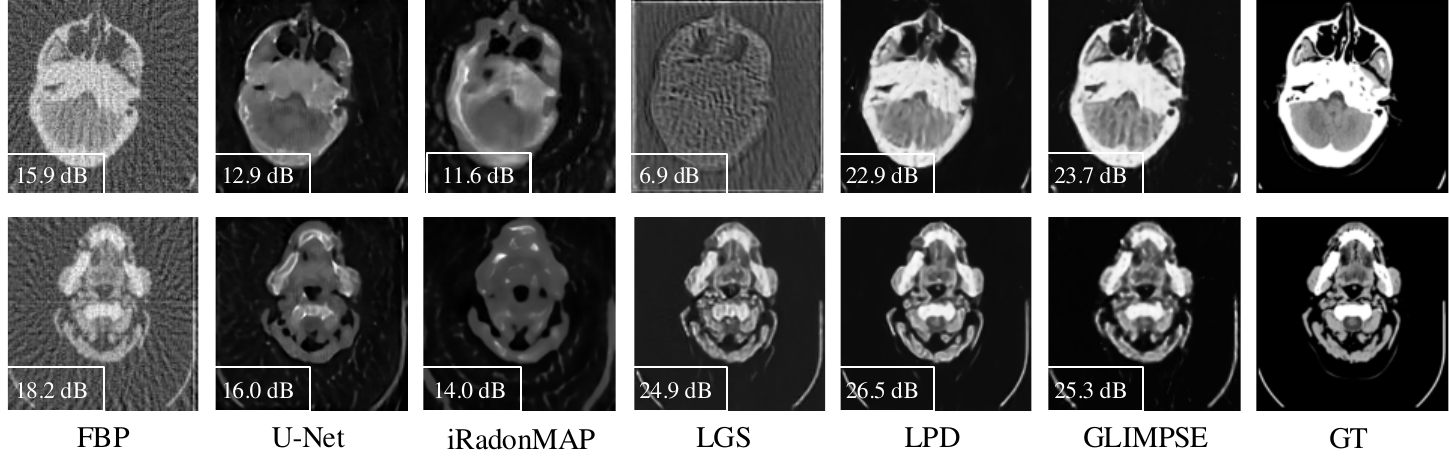}
    \caption{OOD brain samples}
    \label{fig: 128_ood}
    \end{subfigure}
    \caption{Performance of different models trained on training data of chest images and evaluated on in-distribution and OOD samples. {\dlt} shows very strong performance on OOD data, significantly better than U-Net~\cite{ronneberger2015u}, \rev{iRadonMAP \cite{he2020radon}}, LGS~\cite{Adler2017solving} and comparable with LPD~\cite{Adler2018learned}. \rev{We indicate PSNRs between the reconstructions and the ground truth.}}
    \label{fig: 128_results}
    \end{figure*}

\begin{table}
\renewcommand{\arraystretch}{1.3}
    \centering
    \caption{Comparison of different models for sparse view CT. The reconstruction quality is calculated on 64 test samples.}
    \label{tab: quantitative results}
    \resizebox{0.49\textwidth}{!}{%
    \begin{tabular}{c| c |cc|cc}
    \textbf{Methods} & \textbf{Num params} & \multicolumn{2}{c}
    {\textbf{In-distribution (chest)}} & \multicolumn{2}{c}{\textbf{Out-of-distribution (brain)}} \\
 \hline
  & & \textbf{PSNR} & \textbf{SSIM} &  \textbf{PSNR} & \textbf{SSIM} \\
    \hline
    FBP~\cite{feldkamp1984practical} & $0$  & $17.0 \pm 1.9$ & $0.17 \pm 0.06$ & $17.1 \pm 1.3$ & $0.22 \pm 0.02$ \\
    U-Net~\cite{ronneberger2015u} & $7800$k  & $30.1 \pm 1.4$  & $0.84 \pm 0.02$ & $15.1 \pm 1.8$ & $0.28 \pm 0.03$ \\
    \rev{iRadonMAP~\cite{he2020radon}} & $8400$k  & $28.5 \pm 1.3$  & $0.80 \pm 0.03$ & $13.4 \pm 2.1$ & $0.23 \pm 0.06$  \\
    \rev{iRadonMAP-ff} & $8200$k  & $30.1 \pm 1.3$  & $0.83 \pm 0.02$ & $14.2 \pm 1.6$ & $0.25 \pm 0.04$  \\
    \rev{iRadonMAP-ffnu} & $500$k  & $25.2 \pm 1.5$  & $0.64 \pm 0.03$ & $19.5 \pm 1.9$ & $0.36 \pm 0.07$  \\
    LGS~\cite{Adler2017solving} & $19$k & $30.9 \pm 1.4$  & $0.84 \pm 0.02$ & $20.5 \pm 7.7$  & $0.54 \pm 0.31$  \\
    LPD~\cite{Adler2018learned} & $400$k & $\textbf{31.6} \pm 1.4$ & $\textbf{0.86} \pm 0.02$ & $\textbf{25.5} \pm 2.6$  & $0.76 \pm 0.06$   \\ 
     {\dlt} (MLP) & $900$k   & $30.9 \pm 1.4$ &   $0.84 \pm 0.02$ & $25.1 \pm 2.3$ &   $\textbf{0.79} \pm 0.05$  \\
     {\dlt} (MultiMLP)  & $900$k  & $31.0 \pm 1.4$ &   $0.84 \pm 0.02$ & $25.0 \pm 2.3$ &   $0.77 \pm 0.05$  \\
    \hline
    \end{tabular}
    }
\end{table}

We simulate parallel-beam X-ray CT with $r = 30$ projections uniformly distributed around the object with additive Gaussian noise to reach a signal-to-noise ratio (SNR) of 30 dB. Model performance is assessed on 64 in-distribution test samples of chest images, while 16 OOD brain images~\cite{hssayeni2020computed} are included to evaluate the generalization capability of the models.

\revv{{\dlt} (MLP) uses an MLP with 9 hidden layers of dimensions [256, 256, 256, 256, 128, 128, 128, 64, 64], with ReLU activations. {\dlt} (MultiMLP) consists of nine small MLP blocks, each with three hidden layers of size 128. The outputs of these MLPs are then combined using an additional MLP with the same architecture. To ensure a fair comparison, both {\dlt} (MLP) and {\dlt} (MultiMLP) are designed to have a comparable number of trainable parameters.} The input to the MLP network consists of sinusoidal curves sampled from $K = 9^2$ neighboring pixels. To prevent boundary cross talk due to circular convolution (since we implement an unconstrained discrete Fourier transform multiplier), we apply zero-padding with a size of 512 to the sinogram before applying the filter $\h$. Linear interpolation is used in~\eqref{eq: sinogram sampler}.

\subsubsection{Training data of chest images}
\label{sec: chest images training}
We use 35820 training samples of chest images from the LoDoPaB-CT dataset~\cite{leuschner2021lodopab} in resolution $128 \times 128$.
Figure~\ref{fig: 128_test} and Table~\ref{tab: quantitative results} show the performance of different models on in-distribution test samples of chest images. We see that {\dlt} (MLP) and {\dlt} (MultiMLP) outperform successful CNNs like U-Net \rev{and iRadonMAP} and achieve comparable performance with LGS and LPD methods, all while using simple MLPs.

Figure~\ref{fig: 128_ood} and Table~\ref{tab: quantitative results} compare the various models trained on chest images and applied to OOD brain images. This experiment demonstrates that while U-Net, \rev{iRadonMAP and iRadonMAP-ff} excel on in-distribution samples, their performance significantly deteriorates on OOD data.

By contrast, {\dlt} (MLP) shows strong performance on OOD data. \rev{{\dlt} (MultiMLP) achieves comparable performance with {\dlt} (MLP) which showcases the suitability of the new MultiMLP architecture.} Although LPD's performance on OOD data is sometimes comparable or slightly better than that of {\dlt}, it comes at an extremely high memory and compute cost; we analyze this further in Section \ref{sec: computational efficiency}. 

\rev{Table \ref{tab: quantitative results} also highlights the superior performance of {\dlt} compared to iRadonMAP and its variants, particularly iRadonMAP-ffnu, which excludes the post-processing CNN. This can be explained by two key factors: (1) Unlike iRadonMAP, which extracts a single sinusoidal curve per pixel, {\dlt} also processes neighboring pixels, enabling significantly better reconstructions; and (2) while iRadonMAP-ffnu uses a linear transformation for local neighborhood processing, {\dlt} leverages a much more expressive non-linear mapping via MLPs.}

\rev{On the other hand, iRadonMAP and iRadonMAP-ff show better reconstruction on in-distribution chest data but generalize poorly compared to the local processing iRadonMAP-ffnu. This is due to the post-processing CNN in iRadonMAP and iRadonMAP-ff, which negatively impacts generalization. Finally, the filter in iRadonMAP-ff outperforms the MLP filter in the original version, demonstrating the advantage of simple linear filtering, as discussed in Section \ref{sec: adaptive filter}.}

\begin{table}
    \renewcommand{\arraystretch}{1.3}
        \centering
        \caption{Comparison of different models for sparse view CT image reconstruction; the reconstruction quality is calculated on 64 test samples.}
        \label{tab: celeba}
        \resizebox{0.49\textwidth}{!}{%
        \begin{tabular}{c| c |cc|cc}
        \textbf{Datasets} & \textbf{Num samples} & \multicolumn{2}{c}
        {\textbf{Chest}} & \multicolumn{2}{c}{\textbf{Brain}} \\
     \hline
      & & \textbf{PSNR} & \textbf{SSIM} &  \textbf{PSNR} & \textbf{SSIM} \\
        \hline
        Chest \cite{leuschner2021lodopab} & $35820$  & \textbf{30.9 }& $\textbf{0.84}$ & $25.1$ & $0.79$ \\
        DIV2K \cite{agustsson2017ntire} & $800$  & $27.8$  & $0.75$ & $23.3$ & $0.65$ \\
        CelebA-HQ \cite{karras2017progressive} & $30000$  & $28.8$  & $0.79$ & $\textbf{25.3}$ & $\textbf{0.80}$  \\
        \hline
        \end{tabular}
        }
\end{table}

\subsubsection{Training data of natural images}
\label{sec: natrual images training}
\rev{
The robustness of {\dlt} to distribution shift motivates an experiment to examine the impact of the training dataset on performance. For this purpose, we consider two distinct datasets of natural images: (1) DIV2K \cite{agustsson2017ntire}, with 800 high-quality natural images, and (2) CelebA-HQ \cite{karras2017progressive}, with 30,000 high-resolution images of human faces. Except the training dataset, the network architecture and the training details are the same as Section \ref{sec: chest images training}.
Table \ref{tab: celeba} presents the performance of {\dlt} trained on these datasets and applied to chest and brain medical images. Notably, CelebA-HQ, despite being visually unrelated to medical images, trains {\dlt} as effectively as the chest dataset. By contrast, training with a smaller dataset like DIV2K results in a significant drop in reconstruction quality, highlighting the importance of large high-quality data for improving model generalization. 
}

\subsection{Computational Efficiency}
\label{sec: computational efficiency}

\begin{figure*}
    \centering
    \begin{subfigure}{0.5\textwidth}
      \centering
     \includegraphics[width=1.05\textwidth]{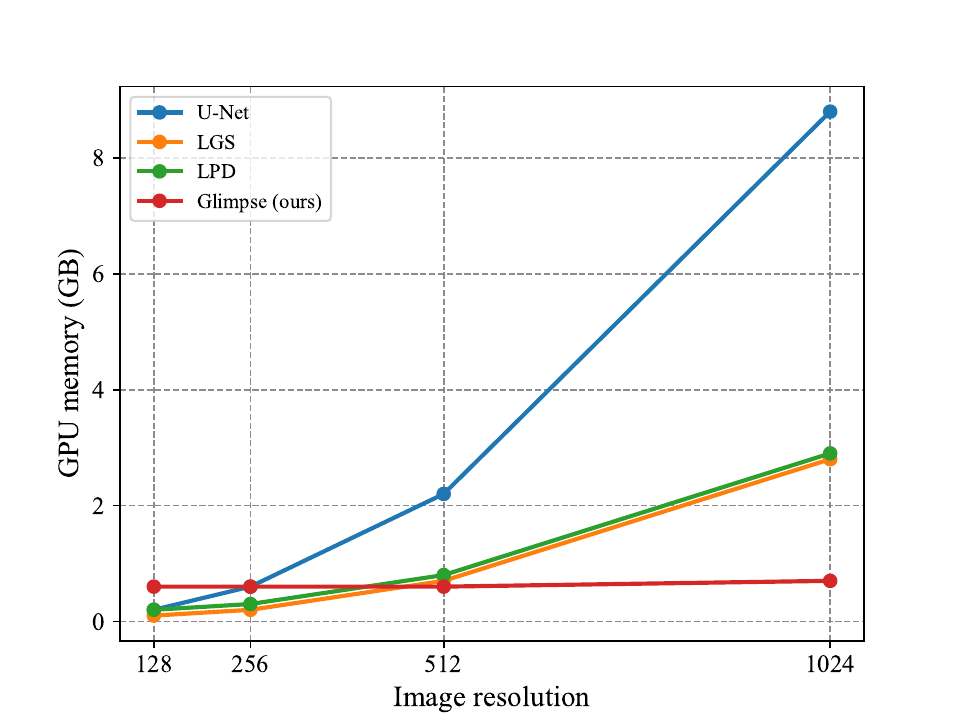}
    \caption{Memory footprint  (10 images)}
    \label{fig: gpu_test}
    \end{subfigure}%
    \begin{subfigure}{0.5\textwidth}
    \centering
    \includegraphics[width=1.05\textwidth]{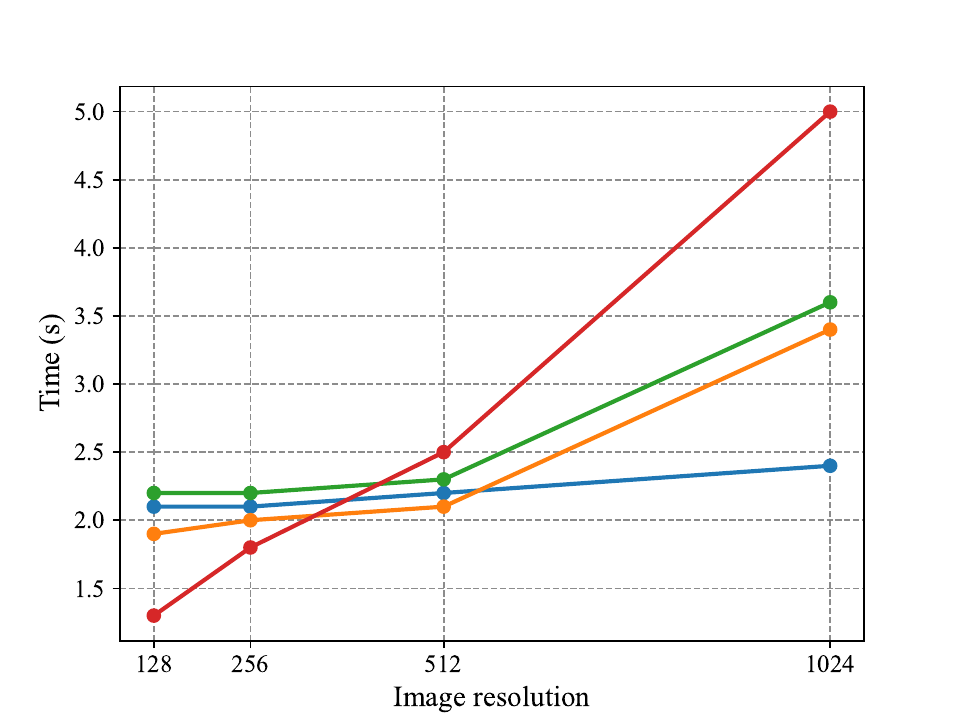}
    \caption{Inference time (10 images)}
    \label{fig: time_test}
    \end{subfigure}
    \caption{The memory and time requirements during inference for different models.}
    \label{fig: inference computation}
\end{figure*}

The fact that LPD far outperforms U-Net on OOD data is a testament to the benefits of incorporating the forward operator in the architecture. \rev{However, evaluating the Radon transform and its adjoint can become prohibitively expensive for large images, as it implies storing multiple copies of the same size as the original image. It can be partially mitigated by reducing the number of iterations in the associated iterative reconstruction scheme but at the cost of a significant deterioration in reconstruction quality. In this section, we compare the training memory and time requirements of different models at different resolutions, for 500 iterations with batch size 64.} We report the maximum use of GPU memory and the time needed to complete the training and inference. As evident from Figure \ref{fig: training computation}, the success of LPD and LGS comes at the cost of very unfavorable training memory and time complexity which rapidly worsens with resolution.
\rev{On the other hand, the memory needed to train {\dlt} is almost independent from image resolution. Remarkably, {\dlt} needs only 5GB memory to train on $1024 \times 1024$ images---less than 1/16 of the memory typically needed
by standard CNNs for CT image reconstruction. This makes {\dlt} suitable for high dimensional reconstruction tasks in real-world applications.}

\rev{We next compare the computational efficiency of various models during inference. With {\dlt}, there is a trade-off between inference speed and memory usage: smaller batch sizes reduce memory consumption but slow down inference, whereas larger batch sizes enable faster inference at the cost of higher memory usage. In this experiment, we set the pixel batch size to 1024. Figure \ref{fig: inference computation} presents the memory footprints and runtimes of different models for reconstructing 10 samples. Although {\dlt} performs pixel-wise image synthesis, it remains comparable to other CNNs that recover the whole image at once. For further discussion on the computational cost and potential remedies, please refer to Section \ref{sec: limitations}.}

Finally, we study the performance of {\dlt} (MultiMLP) on \rev{higher-resolution CT reconstruction. We train on the LoDoPaB-CT dataset at resolution $512 \times 512$, using \revv{90} projections with 40dB measurement noise. \revv{For this experiment, we use a larger MultiMLP with hidden layer dimension 400 to enhance the quality of reconstructions.}} Figure~\ref{fig: 512_results} shows the performance of {\dlt} on in-distribution and OOD samples, \revv{along with the pixel-wise absolute error maps between the reconstructions and ground truth images.} This experiment demonstrates that our proposed framework can achieve strong performance in realistic high resolutions.

\begin{figure*}
\centering
\begin{subfigure}{.5\textwidth}
  \centering
 \includegraphics[width=0.9\textwidth]{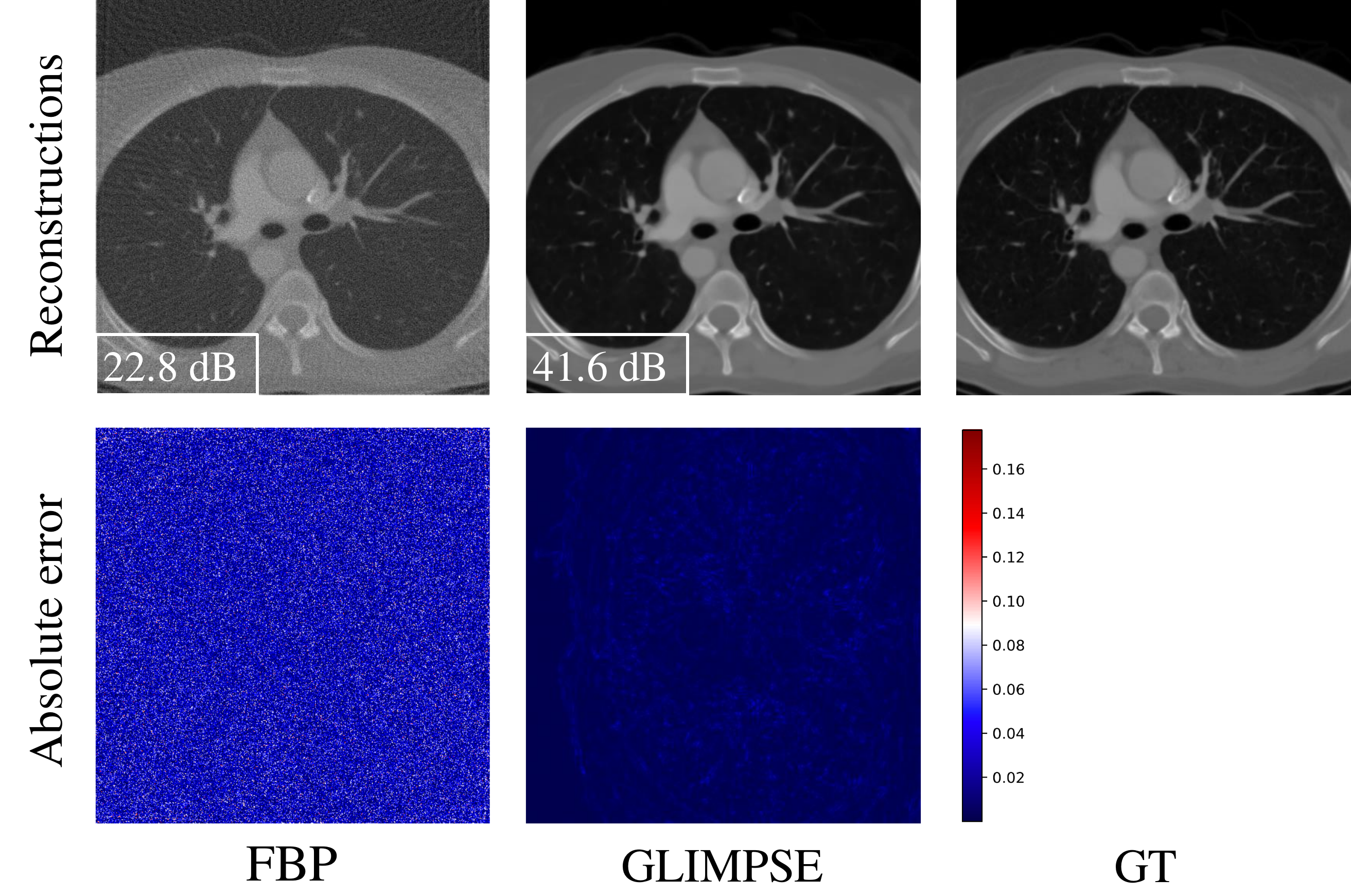}
\caption{In-distribution chest samples}
\label{fig: 512_test}
\end{subfigure}%
\begin{subfigure}{.5\textwidth}
\centering
\includegraphics[width=0.9\textwidth]{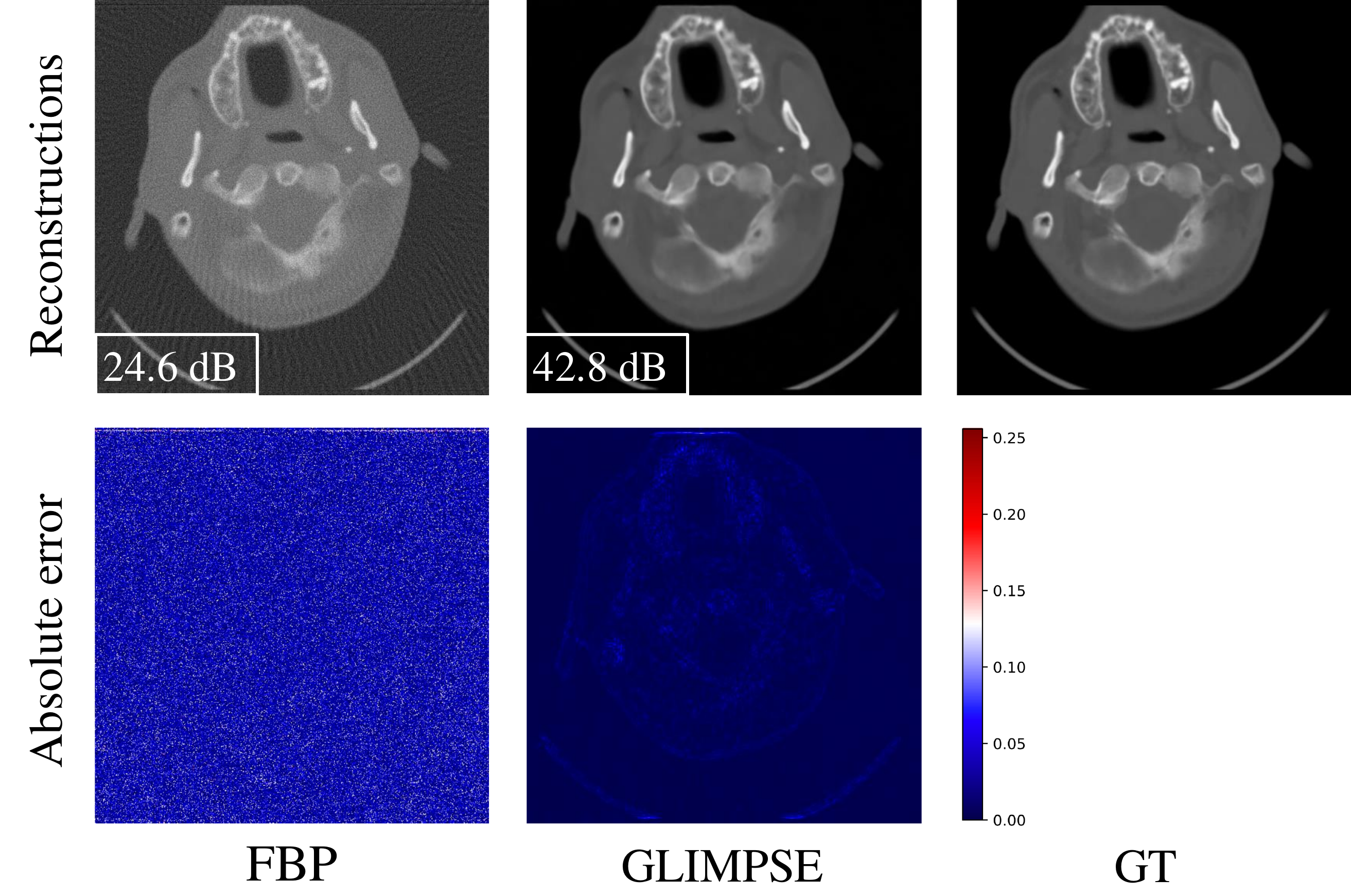}
\caption{OOD brain samples}
\label{fig: 512_ood}
\end{subfigure}
\caption{{\dlt}'s performance in resolution $512 \times 512$ trained on chest training data with $r = 90$ projections and 40dB noise. \rev{We indicate PSNRs between the reconstructions and the ground truth along with the pixel-wise absolute error maps.}}
\label{fig: 512_results}
\end{figure*}

\subsection{Learned Filter}
\label{sec: learned filter}

\begin{figure}
\centering
    \includegraphics[width = 0.5\textwidth]{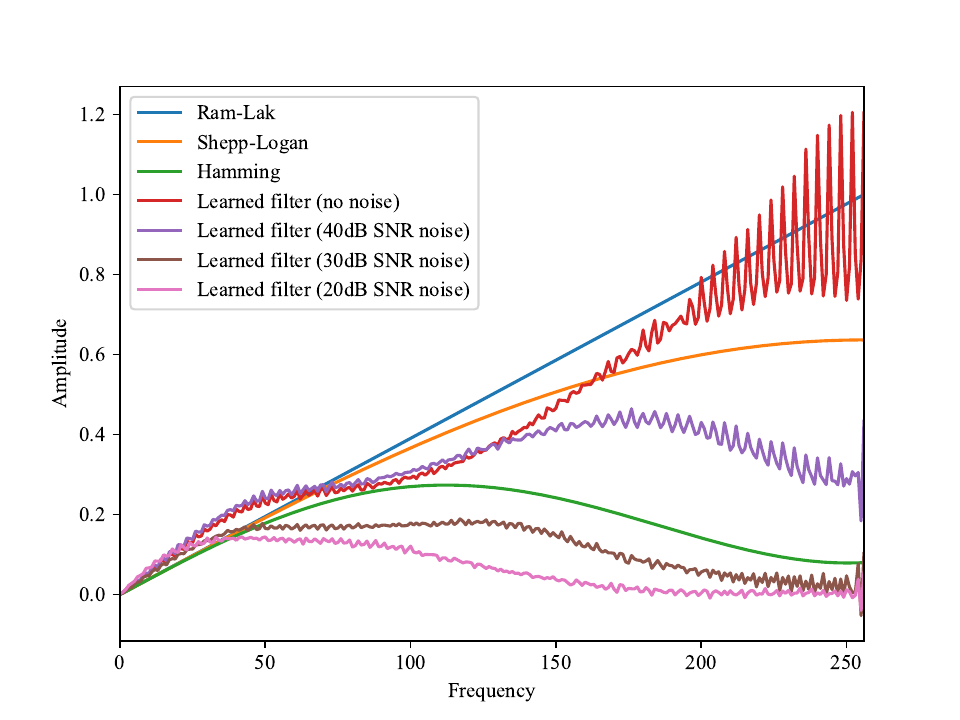}
    \caption{The learned filter for datasets with different noise levels, all the filtered are initialized by Ram-Lak filter in {\dlt} architecture. By increasing the noise level, the filter assigns smaller amplitudes for high-frequencies to suppress the noise and aligns with the optimality of the Ram-Lak filter for noise-free complete measurements shown in Section~\ref{app: optimal filter}.}
    \label{fig: learned filters}
\end{figure}

In this section, we study the learnable filter introduced in Section~\ref{sec: adaptive filter} across datasets with different measurement noise levels. This provides useful signal processing insights into how the properties of the learned filter are influenced by varying noise levels. In Figure~\ref{fig: learned filters} we show the frequency response of the learned filters, alongside standard hand-crafted filters such as Ram-Lak, Shepp-Logan, and Hamming. The learned filters are trained jointly with the MLPs in {\dlt}. As expected (see also the discussion in Appendix~\ref{app: optimal filter}), the learned filter for noise-free measurements is similar to the Ram-Lak filter, with a relatively high amplitude in high frequencies. As the noise level increases (by decreasing the noise SNR), the filter progressively takes smaller values in high frequencies to suppress the noise. This shows that {\dlt} can indeed autonomously adapt the characteristics of the filter according to noise (and other characteristics) in the training data. \rev{We additionally observe that training {\dlt} with a learnable filter leads to much faster convergence compared to a fixed filter (such as the Ram-Lak) while achieving comparable (or slightly better) reconstruction quality. Reconstructed images for different noise levels are presented in Figure~\ref{fig: noise_analysis}.
}

\begin{figure*}
\centering
\begin{subfigure}{.9\textwidth}
  \centering
 \includegraphics[width=0.99\textwidth]{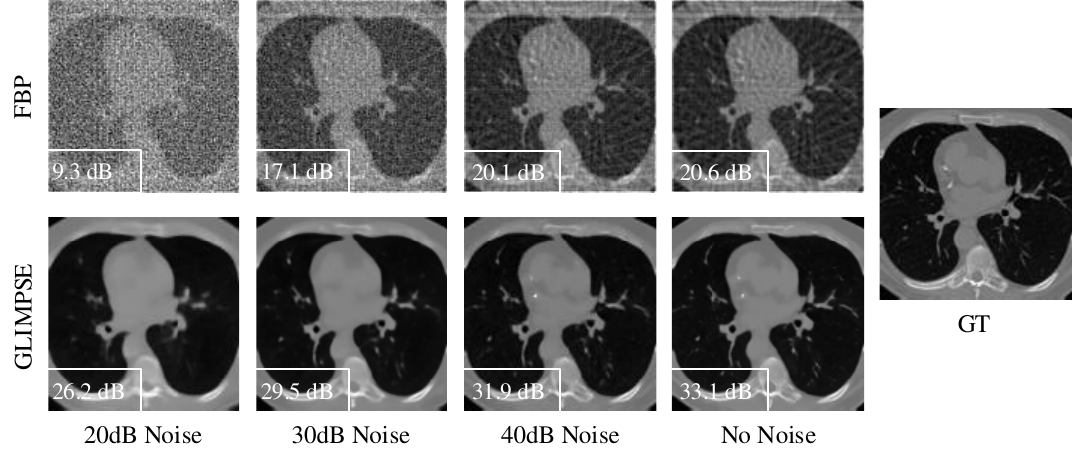}
\caption{In-distribution chest samples}
\label{fig: noise_test}
\end{subfigure}
\\[8mm]
\begin{subfigure}{.9\textwidth}
\centering
\includegraphics[width=0.99\textwidth]{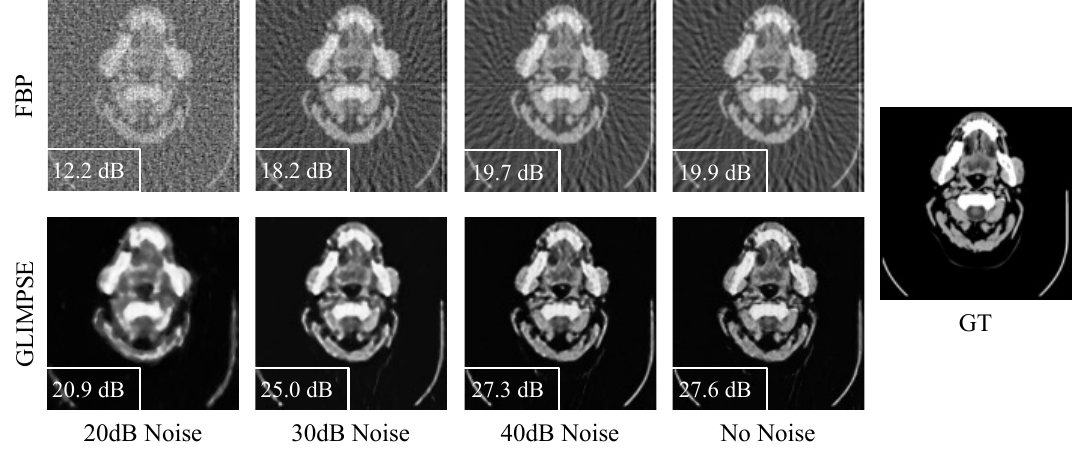}
\caption{Out-of-distribution brain samples}
\label{fig: noise_ood}
\end{subfigure}
\caption{{\dlt} performance on in-distribution and OOD data for different measurement noise levels with $r = 30$ projections. \rev{We indicate PSNRs between the reconstructions and the ground truth.}}
\label{fig: noise_analysis}
\end{figure*}

\rev{
\subsection{Influence of the Number of Projections}
\label{sec: ablation_proj}
As mentioned in Section \ref{sec: MultiMLP}, {\dlt} (MultiMLP) can process measurements with large number of projections $r$. To show the effectivity of the proposed architecture, we study the performance of {\dlt} (MultiMLP) for different number of projections while we have 30dB measurement noise. Separate {\dlt} (MultiMLP) models were trained on datasets with varying numbers of projections. Figure \ref{fig: ablation_proj_analysis} shows the reconstructions for different number of projections.}

\begin{figure*}
    \centering
    \begin{subfigure}{.9\textwidth}
      \centering
     \includegraphics[width=0.99\textwidth]{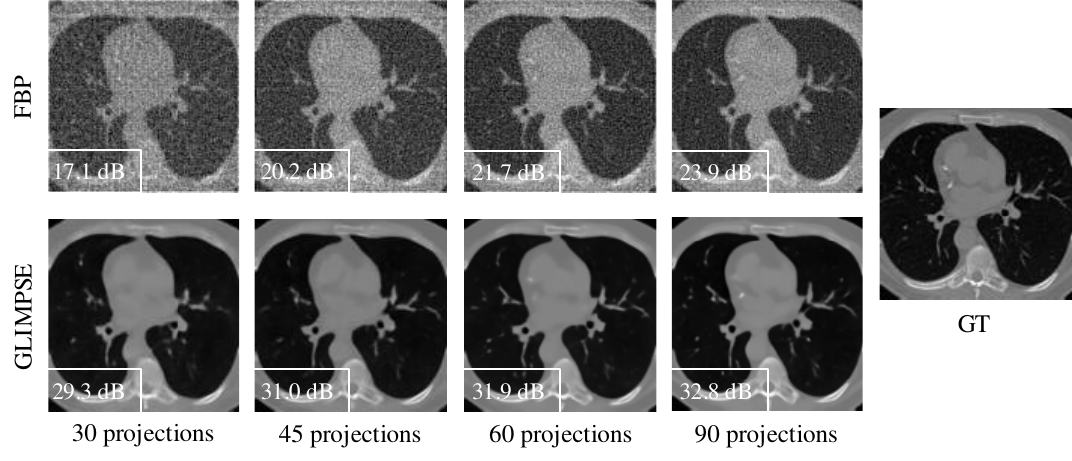}
    \caption{In-distribution chest samples}
    \label{fig: ablation_proj_test}
    \end{subfigure}
    \\[8mm]
    \begin{subfigure}{.9\textwidth}
    \centering
    \includegraphics[width=0.99\textwidth]{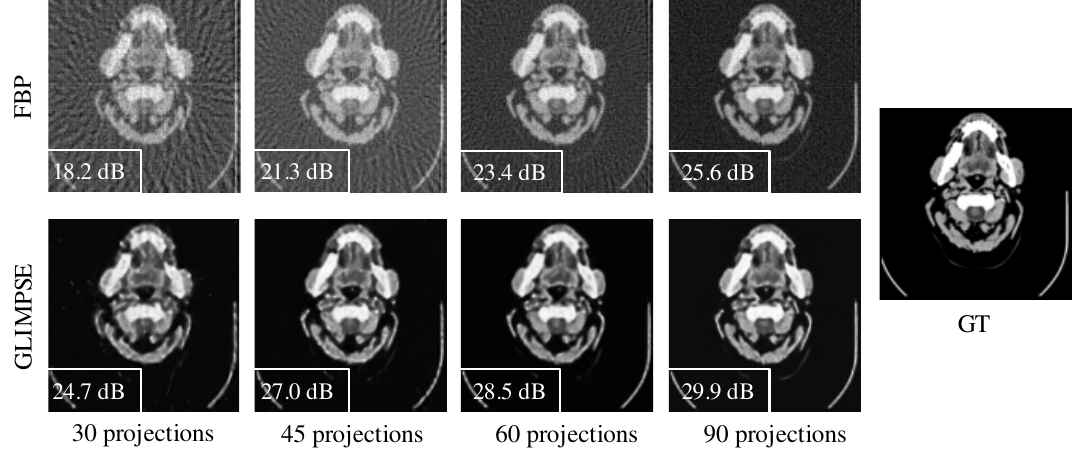}
    \caption{Out-of-distribution brain samples}
    \label{fig: ablation_proj_ood}
    \end{subfigure}
    \caption{{\dlt} performance on in-distribution and OOD data for different number of projections with measurement noise 30dB. \rev{We indicate PSNRs between the reconstructions and the ground truth.}}
    \label{fig: ablation_proj_analysis}
\end{figure*}

\rev{
\subsection{Influence of the Neighborhood Size}
\label{sec: window size}
In this section, we analyze the significance of contextual information on {\dlt}'s performance by varying the number of neighboring pixels (patch size) $K=C^2$. Table~\ref{tab: window size} presents the performance of {\dlt} trained with different patch sizes $K$ on both in-distribution and out-of-distribution (OOD) samples. The results demonstrate that {\dlt} with $K = 3 \times 3$ significantly outperforms the model without contextual information ($K=1$). Moreover, we see that the reconstruction quality tends to reach a saturation point beyond a certain patch size. This observation can inform the optimal choice of context size.}

\begin{table}
\renewcommand{\arraystretch}{1.3}
\caption{Reconstruction quality in PSNR (dB) for {\dlt} trained with various number of neighboring pixels.}
\label{tab: window size}

    \centering
    \resizebox{0.45\textwidth}{!}{%
    \begin{tabular}{@{}l|cccc|c@{}}
    \hline
    Patch size ($K=C^2$)& In-distribution & OOD & Num params \\
    \hline
    $1 \times 1$   &  25.6  &  18.3 & 280k \\
    $3 \times 3$   & 30.2   &  24.2 & 345k \\
    $5 \times 5$   &  30.7  & 24.9  & 470k \\
    $7 \times 7$   &  30.8  & 24.9  & 650k \\
    $9 \times 9$   &  \textbf{30.9}  &  \textbf{25.1}  & 900k \\
    $11 \times 11$   & \textbf{30.9}   &  25.0 & 1200k \\
    \hline
    \end{tabular}}

\end{table}

\section{Discussions and Conclusion}

We have demonstrated that \dlt\ --- a neural network adapted to the geometry of computed tomography---can be much more robust, much more scalable, and much less data hungry CT reconstructions than the leading CNN-based (and model-based) methods. Our experiments substantiate the key claims made in the Introduction. First, by exploiting local sinusoidal patches in the sinogram, \dlt\ handles out-of-distribution data more gracefully than leading CNN-based methods. 
Second, since training is done at the pixel level, \dlt's GPU memory usage remains nearly constant as the image resolution grows, making it scalable to 1024x1024 or higher without requiring prohibitively large hardware. 
Finally, the learnable filter and differentiable projection angles make \dlt\ highly flexible in practice, able to handle noisy datasets and even uncalibrated systems where sensor geometry is only partially known. \rev{This last feat is facilitated by the robustness and numerical efficiency of \dlt.}

\subsection{Limitations}
\label{sec: limitations}
\rev{
{\dlt} can be trained on GPUs with significantly smaller memory than baselines, which enables very high-dimensional image reconstruction, but its computational cost at inference scales with the number of pixels. Recent work \cite{he2023dynamic, he2024latent} has improved the efficiency of continuous image representation in INRs by increasing shared computations across coordinates, thereby reducing computational complexity. Adapting these methods within {\dlt} could potentially decrease inference time. We note, however, that even with the current architecture inference is essentially real-time.
}

Another challenge is that memory and compute cost increase with the number of projections $r$. A possible alternative to the standard MLP or MultiMLP architectures which are the culprit for this is to use mixture-of-experts layers \cite{Shazeer2017, Riquelme2021scaling, Fedus2022review}, which selectively employ smaller MLPs for processing inputs. This approach is an effective drop-in replacement for standard MLP layers of language transformers \cite{vaswani2017attention} and vision transformers \cite{dosovitskiy2020image}; we leave it to future work to test its effectiveness in local CT reconstruction.

\revv{Since the dimensionality of the MLP network is fixed, {\dlt} can only process data with the specific number of projections it was trained on. This limitation is common in most deep-learning models for tomographic reconstruction, including model-based architectures like LPD and LGS. Here, however, it arises specifically from the MLP structure. Architectures such as transformers \cite{vaswani2017attention}, which can process data sequentially, are likely the right solution. }

\subsection{Looking forward: locality for other imaging modalities}
\label{sec: future works}
\rev{{\dlt} can be generalized to various imaging problems where the forward operator involves line integrals, such as fan-beam computed tomography (CT) \cite{kak2001principles}. In fan-beam CT, X-rays diverge from a source point in a fan-shaped pattern as they pass through the object, a configuration commonly used in clinical CT scanners due to its efficiency in capturing larger areas. As detailed in \cite[\S 5.11.6]{gonzalez2006digital}, although the fan-beam CT forward operator is more complex than that of parallel-beam CT, it retains a local structure that can be exploited to develop a local processing reconstruction pipeline, similar to {\dlt}.
}
{\dlt} can also be extended to other imaging modalities with a local forward operator including photoacoustic~\cite{jathoul2015deep,yao2015high} and cryo-electron tomography (cryoET)~\cite{doerr2017cryo,debarnot2024ice}. Its future full-3D adaptation may yield efficient architectures that resolve the fundamental memory issues with applications of deep learning in 3D medical imaging. 
This extension is particularly interesting given the ability of {\dlt} to operate locally and its near-fixed memory requirement across resolution, which makes it a potentially strong choice for full 3D problems.


\appendix


\subsection{Learned Sensor Geometry}
\label{app: learnable geometry}

CT imaging algorithms such as  FBP~\cite{feldkamp1984practical}, SART~\cite{andersen1984simultaneous}, LGS~\cite{Adler2017solving}, LPD~\cite{Adler2018learned} assume that the projection angles $\{\alpha_m\}_{m=1}^r$ are known. In an uncalibrated system where sensor geometry is different from what the algorithms assume, the quality of reconstruction deteriorates~\cite{lunz2021learned, hauptmann2023model}. {\dlt} allows directly optimizing the projection angles during training. We thus jointly optimize $\{\alpha_m\}_{m = 1}^r$ with other trainable parameters in \eqref{eq: training}. This additional angle estimation incurs a very modest computational cost. 

In the absence of calibration, we cannot expect to have paired ground truth images. In the following experiments, we only want to showcase the possibility to differentiably optimize over angles in {\dlt} so we assume having access to paired data (while simulating the uncalibrated forward operator). In practice, we could use a self-supervised loss, for example, based on equivariance~\cite{chen2021equivariant}.

\begin{figure}
\centering
    \includegraphics[width = 0.45\textwidth]{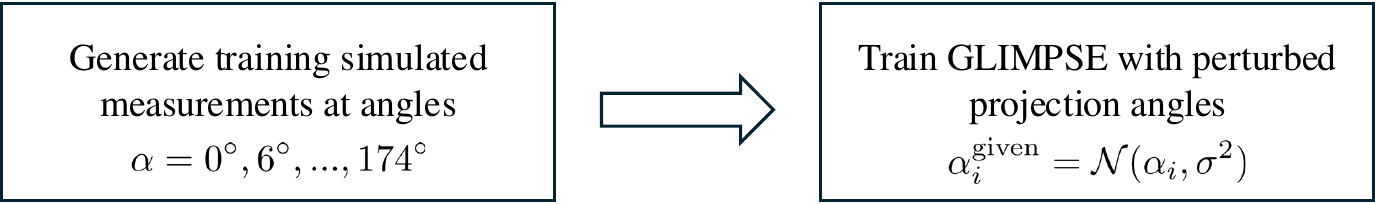}
    \caption{The experimental arrangement for conducting uncalibrated imaging experiments}
    \label{fig: uncalibrated setup}
\end{figure}

\begin{figure*}
\centering
\begin{subfigure}{.45\textwidth}
  \centering
 \includegraphics[width=0.99\textwidth]{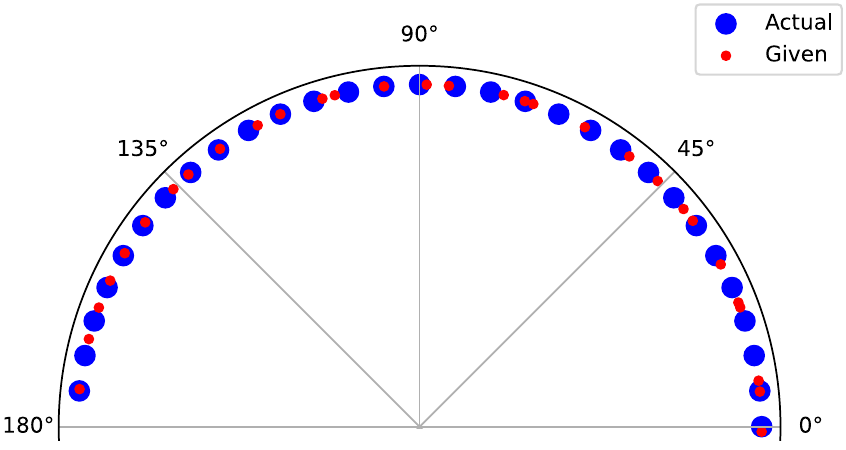}
\caption{Given projection angles}
\label{fig: given_sensors_uncalibrated_random}
\end{subfigure}%
\begin{subfigure}{.45\textwidth}
\centering
\includegraphics[width=0.99\textwidth]{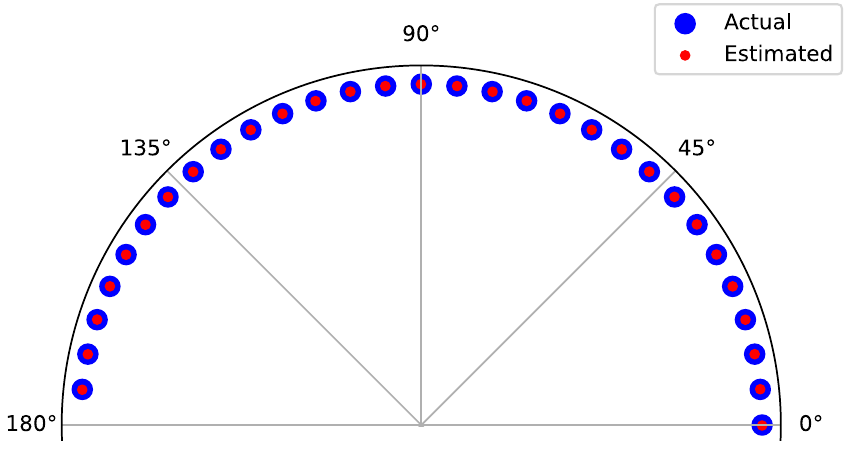}
\caption{Estimated projection angles}
\label{fig: estimated_sensors_uncalibrated_random}
\end{subfigure}
\\[5mm]
\begin{subfigure}{0.8\textwidth}
\centering
\includegraphics[width=\textwidth]{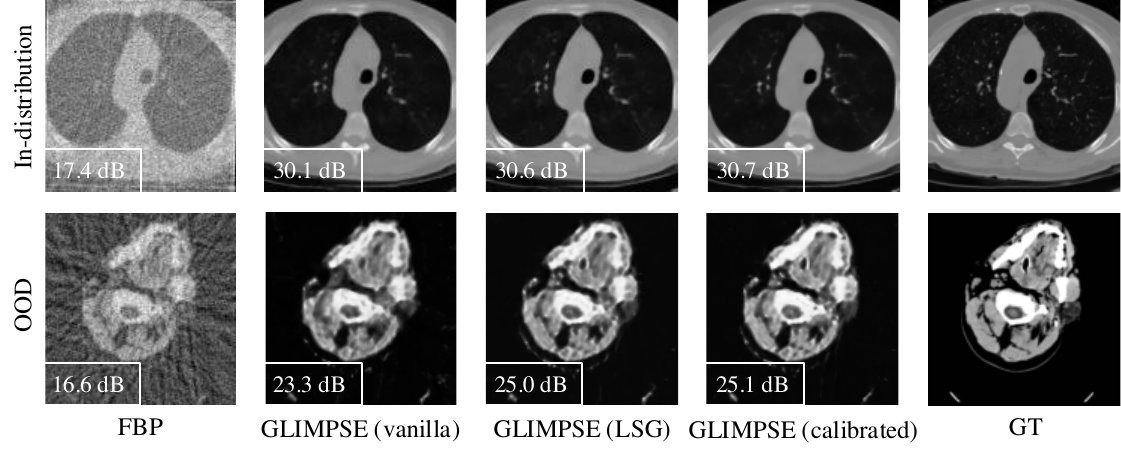}
\caption{Reconstructions}
\label{fig: reconstructions_uncalibrated_random}
\end{subfigure}
\caption{Estimated sensor geometry by {\dlt} (LSG) and reconstructions for an uncalibrated system with a random sensor shift; as expected, the learnable sensor geometry can effectively learn the projection angles and exhibits excellent robustness with no degradation under such a big model mismatch and measurement noise (30dB). \rev{We indicate PSNRs between the reconstructions and the ground truth.}}
\label{fig:sensors-locs-random-shift}
\end{figure*}

We assess the performance of {\dlt} in situations with mismatched projection orientations. In the following experiments, we place $r = 30$ sensors uniformly around the object at angles $\alpha = 0^\circ,6^\circ,...,174^\circ$.
We compare three models: 1) {\dlt} (vanilla), with no learnable sensor geometry, 2) {\dlt} (LSG), incorporating the proposed learned sensor geometry, and 3) {\dlt} (calibrated), operating under ideal conditions with no model mismatch (informed with correct projection angles).
\rev{Figure \ref{fig: uncalibrated setup} demonstrates the experimental procedure for  uncalibrated imaging experiments.}
We let the {\dlt} (LSG) learn the projection angles from the training data where the optimized values $\{\alpha_m\}_{i=1}^r$ obtained through training can provide a reliable estimate of the actual projection angles.

\subsubsection{Uncalibrated system with random sensor shifts}
\label{sec: uncal_random}
As shown in Figure~\ref{fig: given_sensors_uncalibrated_random}, we randomly perturb projection angles by a normally distributed error so that $\alpha_i^\text{given} = \mathcal{N}(\alpha_i, \sigma^2)$; we set $\sigma = 2^\circ$. \rev{We train {\dlt} (vanilla) on this uncalibrated dataset; despite this mismatch in the forward operator, {\dlt} (vanilla) can still generate high-quality reconstructions for in-distribution test data (only 0.6 dB drop compared to the calibrated system) as shown in the first row of the second column in Figure \ref{fig: reconstructions_uncalibrated_random}. However, the mismatch in the forward operator does not allow {\dlt} (vanilla) to generalize well on OOD data (1.8 dB drop compared to the calibrated system) as shown in the second row of the second column in Figure \ref{fig: reconstructions_uncalibrated_random}. To address this issue,} we initialize the projection angles $\{\alpha_m\}_{i = 1}^r$ in the {\dlt} (LSG) architecture with $\alpha_i^\text{given}$.
Figure~\ref{fig: estimated_sensors_uncalibrated_random} shows the estimated projection angles obtained through training---{\dlt} (LSG) accurately recovers the angles even in the presence of 30 dB measurement noise. As shown in Figure~\ref{fig: reconstructions_uncalibrated_random}, this accurate estimation of projection angles results in high-quality reconstructions by {\dlt} (LSG) comparable with the network trained in an ideal calibrated system.

\subsubsection{Blind inversion with no information from projection angles}
\label{sec: uncal_blind}
We consider the blind scenario where the model operates without any prior knowledge of the sensor geometry making inversion challenging. As shown in Figure~\ref{fig: given_sensors_random}, we initialize the projection angles $\{\alpha_m\}_{i = 1}^r$ in the {\dlt} (LSG) architecture with random values.
The estimated projection angles are shown in Figure~\ref{fig: estimated_sensors_blind}, highlighting {\dlt} (LSG)'s ability for data-driven sensor geometry estimation. 
Figure~\ref{fig: blind_results} presents the reconstructions achieved by {\dlt} in both its vanilla and LSG versions. As expected, FBP and the {\dlt} (vanilla) show poor reconstructions due to the missing sensor geometry information. On the other hand, {\dlt} (LSG) could accurately reconstruct both in-distribution and OOD samples. Remarkably, these results are comparable to those achieved by the calibrated {\dlt} with informed projection angles.

\begin{figure*}
\centering
\begin{subfigure}{.45\textwidth}
  \centering
 \includegraphics[width=0.99\textwidth]{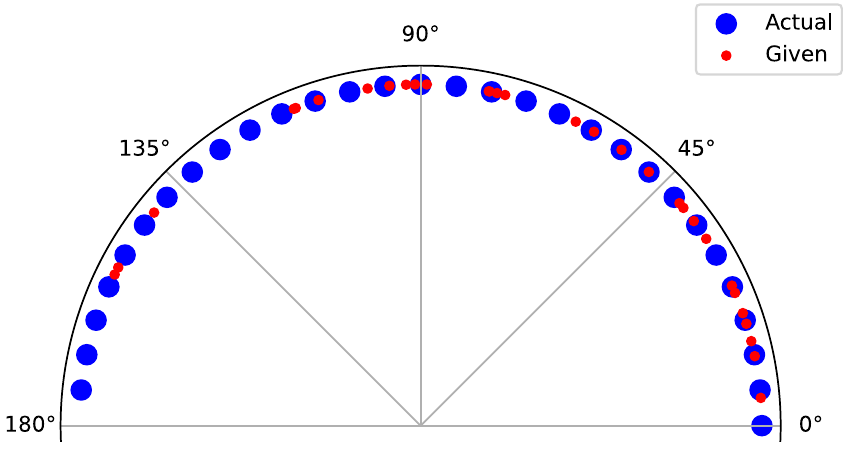}
\caption{Given projection angles}
\label{fig: given_sensors_random}
\end{subfigure}%
\begin{subfigure}{.45\textwidth}
\centering
\includegraphics[width=0.99\textwidth]{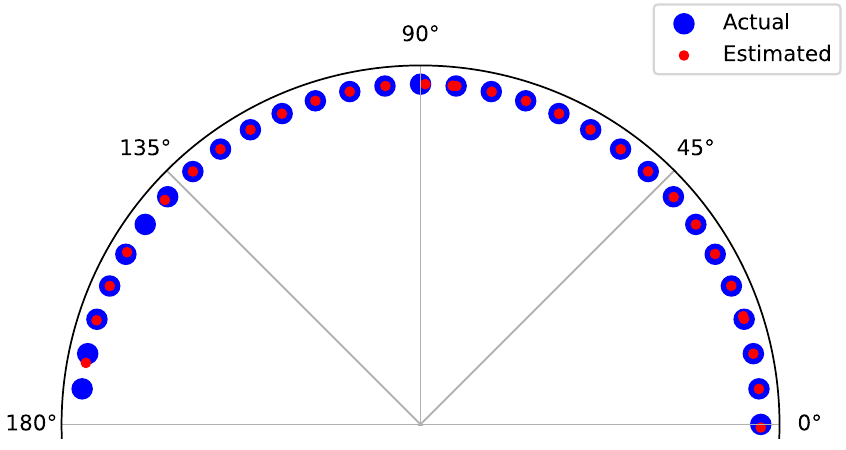}
\caption{Estimated projection angles}
\label{fig: estimated_sensors_blind}
\end{subfigure}
\\[5mm]
\begin{subfigure}{\textwidth}
\centering
\includegraphics[width=0.8\textwidth]{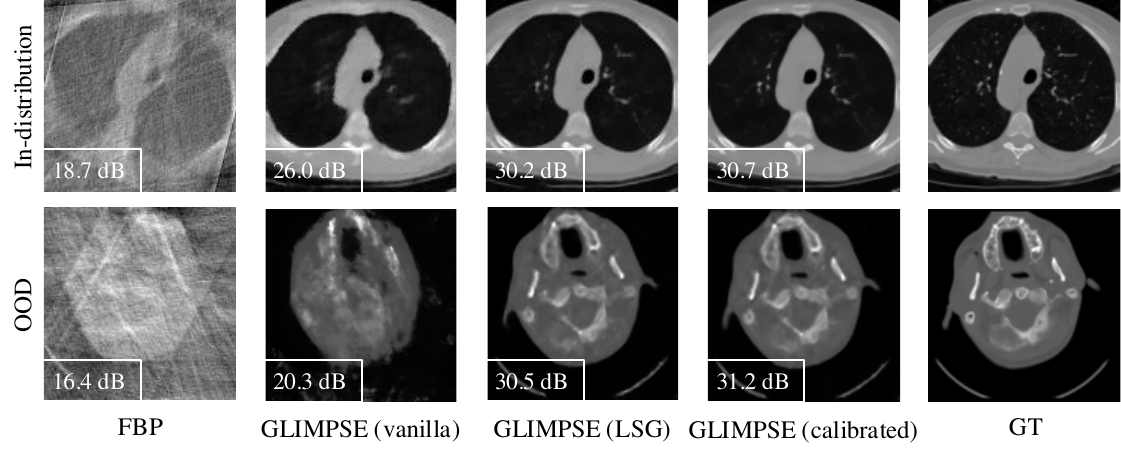}
\caption{High-quality reconstructions by {\dlt} (LSG) despite having no information from sensor geometry.}
\label{fig: blind_results}
\end{subfigure}
\caption{Estimated sensor geometry by {\dlt} (LSG) and reconstructions for blind inversion; {\dlt} (LSG) was initialized with random projection angles $\{\alpha_m\}_{i = 1}^r$ (a) could reliably estimate the projection angles purely from data (b) resulting in high-quality reconstructions (c). \rev{We indicate PSNRs between the reconstructions and the ground truth.}}
\label{fig: sensors_locs_blind}
\end{figure*}

\subsection{Optimal Filter for FBP Reconstruction}\label{app: optimal filter}

\begin{proposition}[Reconstruction for continuous Radon transform]\label{prop:optimal_filter}
We have the following identity
    \begin{align*}
    f(x,y)= \int_{0}^{\pi} Rf(\theta,\cdot)\star \psi d\theta,
\end{align*}
where $\psi$ is the filter that has for Fourier transform $|\cdot|$.
\end{proposition}
\begin{proof}
Let $\p=(x,y)$, $\xib = (\xi_1,\xi_2)$. We have
\begin{align*}
    f(x,y) = & \int_{-\infty}^{+\infty}\int_{-\infty}^{+\infty} \mathcal{F}_{2D}(f)(\xi_1,\xi_2) \exp(2i\pi \langle \xib,\p \rangle )d\xib\\
    = & \int_{0}^{+\infty}\int_{0}^{2\pi} \mathcal{F}_{2D}(f)(r\cos(\theta),r\sin(\theta)) \\
    & \exp(2i\pi r \langle \kk,\p \rangle )rdrd\theta,
\end{align*}
by doing a change of variable in polar coordinates, where $\kk=(\cos(\theta),\sin(\theta)).$
Observe that $\mathcal{F}_{2D}(f)(r\cos(\theta),r\sin(\theta))$ is the Fourier Transform of $f$ along the line of direction $\kk$. By the Fourier slice theorem~\cite{kak2001principles}, we have 
\begin{align*}
    \mathcal{F}_{2D}(f)(r\cos(\theta),r\sin(\theta)) = \mathcal{F}_{1D}(Rf(\theta,\cdot))(r)
\end{align*}
By symmetry of the Radon transform, we have $Rf(\theta,r)=Rf(\theta+\pi,-r)$.
Finally, 
\begin{align*}
    f(x,y) = & \int_{-\infty}^{+\infty}\int_{0}^{\pi} \mathcal{F}_{1D}(Rf(\theta,\cdot))(r) \exp(2i\pi r \langle \kk,\p \rangle ) \\ & |r|drd\theta
    =  \int_{0}^{\pi} \mathcal{F}_{1D}^{-1}\left(\mathcal{F}_{1D}(Rf(\theta,\cdot))\odot |\cdot|\right) d\theta.
\end{align*}
This shows that 

\begin{align*}
    f(x,y)= \int_{0}^{\pi} \left(Rf(\theta,\cdot)\star \psi\right)(\langle \kk,\p\rangle) d\theta,
\end{align*}
where $\psi$ is the filter that has for Fourier transform $|\cdot|$.

\end{proof}

\subsection{Proof of Proposition \ref{prop:sine}}\label{app:proof_sine}
\begin{proof}
Using the definition of the Radon transform in \eqref{eq:radon_def}, we have
\begin{align*}
  Rf(\alpha, t) = \int_{-\infty}^{+\infty}\delta(& z\cos(\alpha) - t\sin(\alpha)-x,\\
  &z\sin(\alpha)+t\cos(\alpha) - y)dz.
\end{align*}
Solving $ z\cos(\alpha) - t\sin(\alpha)-x=0$ for $z$ leads to $$z= \frac{t\sin(\alpha)+x}{\cos(\alpha)}.$$
Then, solving $z\sin(\alpha)+t\cos(\alpha) - y = 0$ for $t$, using the previous expression for $z$ leads to $$t = y \cos(\alpha) - x\sin(\alpha).$$    
\end{proof}

\bibliographystyle{IEEEtran}
\bibliography{arx}

\end{document}